\documentclass[lettersize,journal]{IEEEtran}
\usepackage{amsmath,amsfonts,amssymb,amsthm,thmtools}
\usepackage{algorithmic}
\usepackage{algorithm}
\usepackage{array}
\usepackage[caption=false,font=normalsize,labelfont=sf,textfont=sf]{subfig}
\usepackage{textcomp}
\usepackage{stfloats}
\usepackage{url}
\usepackage{verbatim}
\usepackage{graphicx}
\usepackage{cite}
\usepackage{color}
\usepackage{xcolor}
\definecolor{ccr}{RGB}{0,0,255} 
\definecolor{kkk}{RGB}{0,0,0}
\usepackage[colorlinks,linkcolor=ccr,anchorcolor=ccr,citecolor=ccr,urlcolor=ccr]{hyperref}
\usepackage{orcidlink}
\usepackage{titlesec}
\allowdisplaybreaks[4]
\titlespacing{\section}{0pt}{1ex plus 0.2ex minus .2ex}{1ex plus .2ex}
\titlespacing{\subsection}{0pt}{1ex plus 0ex minus .2ex}{1ex plus .2ex}
\begin{document}
\title{Euclidean Distance Matrix Completion via Asymmetric Projected Gradient Descent
	\vspace{-0.2em} 
}

\author{ 
	Yicheng Li\hspace{0.5mm}
	\orcidlink{0009-0005-7973-4182}, 
	Xinghua Sun\hspace{0.5mm}
	\orcidlink{0000-0003-0621-1469}
	,~\IEEEmembership{Member,~IEEE}
	\vspace{-2em} 
	\thanks{Part of this work\cite{YCLiSunEDMC_Spec_init} has been presented during the 2024 Asia Pacific Signal and Information Processing Association Annual Summit and Conference (APSIPA ASC), Macao, China. This work has been submitted to the IEEE for possible publication. Copyright may be transferred without notice, after which this version may no longer be accessible.}
	\thanks{Y. Li and X. Sun are with the School of Electronics and Communication Engineering, Shenzhen Campus of Sun Yat-sen University, Shenzhen 518107, China (e-mail:\url{liych75@mail2.sysu.edu.cn};\url{sunxinghua@mail.sysu.edu.cn}).}
}

\markboth{}
{Li \MakeLowercase{\textit{et al.}}: Euclidean Distance Matrix Completion via Asymmetric Projected Gradient Descent}


\maketitle
\begin{abstract}
This paper proposes and analyzes a gradient-type algorithm based on Burer-Monteiro factorization, called the Asymmetric Projected Gradient Descent (APGD), for reconstructing the point set configuration from partial Euclidean distance measurements, known as the Euclidean Distance Matrix Completion (EDMC) problem. By paralleling the incoherence matrix completion framework, we show for the first time that global convergence guarantee with exact recovery of this routine can be established given $\mathcal{O}(\mu^2 r^3 \kappa^2 n \log n)$ Bernoulli random observations without any sample splitting. Unlike leveraging the tangent space Restricted Isometry Property (RIP) and local curvature of the low-rank embedding manifold in some very recent works, our proof provides extra upper bounds that act as analogies of the random graph lemma under EDMC setting. The APGD works surprisingly well and numerical experiments demonstrate exact linear convergence behavior in rich-sample regions yet deteriorates rapidly when compared with the performance obtained by optimizing the s-stress function, i.e., the standard but unexplained non-convex approach for EDMC, if the sample size is limited. While virtually matching our theoretical prediction, this unusual phenomenon might indicate that: (i) the power of implicit regularization is weakened when specified in the APGD case; (ii) the stabilization of such new gradient direction requires substantially more samples than the information-theoretic limit would suggest.
\end{abstract}


\section{Introduction}
\IEEEPARstart{G}{iven} $n$ points $\{\mathbf{p}_i^{\star}\}_{i=1}^n$ embedded in $\mathbb{R}^r$, $r=2,3$, calculating their inter-point distances is a trivial task. Forwarding the perfect, complete $n(n-1)/2$ distances back to the point set configuration is, meanwhile well-known as the classic Multi-dimensional Scaling (cMDS)\cite[Sec. 3.1.1]{VanLaurensDR_review}\cite{MDSSurvey}\cite[Thm. 2]{DokmanicEDMTheory}. However, significant effort has been devoted to handle the case when part of the distance measurements are missing, known as the Euclidean Distance Matrix Completion (EDMC) problem\cite{LibertiEDMSurvey}\cite{DokmanicEDMTheory}. This is not only due to its ubiquitous presence in, e.g., computational chemistry\cite{CrippenEneemb}\cite{HendricksonMoleculeProblem}\cite{CucuringuEigSynSinger}\cite{LeungTohSDPDCMP}, operating wireless sensor networks\cite{BiswasSDRSNLacm}\cite{BiswasTR_sdp}\cite{JavanmardMontanariTrEDMC} and articulated robots\cite[Sec. 4.3.2]{LibertiEDMSurvey}\cite{MaricRieOptIKRobot}, channel estimation\cite{AgostiniChannelCEDMC}, indoor SLAM\cite{TabaghiKineticEDM}, ultrasound calibration\cite{ParhizkarOptSpace}, but also the fact that this inverse problem is NP-hard to solve in general\cite[Sec. 3.4]{LibertiEDMDC}. To circumvent the NP-hardness, one may only consider a selected region of the ground truth (generic\cite{KirayTheranACombMC}\cite{ShapiroMCDPGeoPre}, incoherence\cite{TasissaEDMCProof}\cite{KirayCoherenceRec}), enforcing certain connectivity of the underlying sample graph (Erd\"{o}s-R\'{e}nyi\cite{YCLiSunSNLRR}\cite{SmithCaiTasissaRieEDMC2}, random geometric graph\cite{JavanmardMontanariTrEDMC}\cite{OhMontanariMDS_MAP}\cite{MontanariOhSEDMCOptSpace}\cite{KarbasiOhMAPMDSTri}), and posing constrains on certain geometry of the point set ($r$-unique localizable\cite{SoAnthonySNL2006}, trilateration\cite{NetworkLoc1}, universal rigidity\cite{ZhuSoUniverRigid}). Therefore, performance guarantee results in this field can be roughly divided into two classes, i.e., combinatorial\cite{SoAnthonySNL2006}\cite{ZhuSoUniverRigid}\cite{SoYeSDPTensegrity} and probabilistic. While the latter can be further categorized into the traditional yet effective view from random geometry\cite{JavanmardMontanariTrEDMC}\cite{OhMontanariMDS_MAP}\cite{MontanariOhSEDMCOptSpace}\cite{KarbasiOhMAPMDSTri}, and the recent revisit of EDMC via incoherent matrix completion lens\cite{TasissaEDMCProof}\cite{SmithCaiTasissaRieEDMC2}\cite{GhoshTasissaIRLSEDMC}\cite{YCLiSunSNLRR}. Method proposed in this manuscript belongs to the last class. In short, we analyze an algorithm analogous to the Iterative Fast Hard Thresholding for EDMC (IFHT-EDMC) that appeared in a very recent work \cite{SmithCaiTasissaRieEDMC2}, under standard incoherence assumption\cite{ChenIncoOptimalMC} and Bernoulli sampling scheme, from the quotient geometry perspective. 
\subsection{Preliminaries and Motivations}
\label{math_setup}
A Euclidean Distance Matrix (EDM), $\mathbf{D}^{\star}\in\mathbb{R}^{n\times n}$, is referred to be a symmetric hollow diagonal matrix whose $(i,j)$\footnote{We will frequently use boldface Greek letter $\boldsymbol{\alpha}$ or $\boldsymbol{\beta}$ to denote an index tuple $(i,j)\in\mathbb{I}$, as inherited from\cite{TasissaEDMCProof}.} entity denotes the squared distance between $(i,j)$ points
\begin{equation*}
	\setlength\belowdisplayskip{3pt}
	\setlength\abovedisplayskip{3pt}
	\mathbf{D}^{\star}_{ij} := \Vert\mathbf{p}_i^{\star}-\mathbf{p}_j^{\star}\Vert_2^2 = \langle\mathbf{P}^{\star}\mathbf{P}^{\star T},\boldsymbol{\omega}_{\boldsymbol{\alpha}}\rangle = \langle\mathbf{G}^{\star},\boldsymbol{\omega}_{\boldsymbol{\alpha}}\rangle,
\end{equation*}
where $\langle\mathbf{A},\mathbf{B}\rangle:=\mathrm{tr}(\mathbf{A}^T\mathbf{B})$ is the Frobenius inner product.  $\mathbf{P}^{\star}:=[\mathbf{p}_1^{\star},...,\mathbf{p}_n^{\star}]^T\in\mathbb{R}^{n\times r}$ and $\mathbf{G}^{\star}:=\mathbf{P}^{\star}\mathbf{P}^{\star T}$ are the so-called point set configuration and Gram matrix\footnote{The point set configuration is assumed to have full column rank $r$.}, respectively. Let $\mathbf{e}_i$ denote the $i$-th canonical Euclidean basis in $\mathbb{R}^n$, then $\boldsymbol{\omega}_{\boldsymbol{\alpha}}:=(\mathbf{e}_i-\mathbf{e}_j)(\mathbf{e}_i-\mathbf{e}_j)^T$ stands for the (primal) EDM basis. The EDMC problem considered here aims at recovering an EDM from its partial and element-wise measurements, which is usually referred to be the following non-convex QCQP (Quadratically Constrained Quadratic Programming)
\begin{equation}
	\setlength\belowdisplayskip{3pt}
	\setlength\abovedisplayskip{3pt}
	\label{eq_SDR_EDMC_prob}
	\begin{aligned}
		\mathrm{find}\,& \mathbf{G}\\
		\mathrm{s.t.}\,& \langle\mathbf{G},\boldsymbol{\omega}_{\boldsymbol{\alpha}}\rangle = \mathbf{D}^{\star}_{\boldsymbol{\alpha}},\,\boldsymbol{\alpha}=(i,j)\in\Omega,\\
		& \mathbf{G1}=\mathbf{0},\,\mathbf{G}\succcurlyeq\mathbf{0},\,\mathrm{rank}(\mathbf{G})=r,
	\end{aligned}
\end{equation}
where $\mathbf{D}^{\star}_{\boldsymbol{\alpha}}:=\mathbf{D}^{\star}_{ij}$. Superscript $(\cdot)^{\star}$ is used to emphasize the fixed, ground truth that independent with the sample set $\Omega\subset\mathbb{I}:=\{(i,j):1\leq i<j\leq n\}$\footnote{Since the EDM is a symmetric and hollow matrix, we can omit its diagonal and only sample the upper triangle sub-matrix, which is referred to as the symmetric sampling scheme later.}, and the distances between $(i,j)$ points, $(i,j)\in \Omega$, are known. Since rigid motions in $\mathbb{R}^r$ preserve the relative distance between points, the self-centered constraint $\mathbf{G1}=\mathbf{0}$, is usually adopted to remove the translation ambiguity\cite[Eq. (29)]{DokmanicEDMTheory}, where $\mathbf{1}$, $\mathbf{0}$ are the vectors of all ones, all zeros, respectively. 

If the underlying ground truth framework (point set configuration together with the sample graph $\Omega$) is universally rigid, then the rank constraint can be removed at no cost, resulting in streamlined processing on \eqref{eq_SDR_EDMC_prob} with modern convex solvers\cite{ZhuSoUniverRigid}. However, two unresolved issues remain and can not be compromised\cite{YCLiSunSNLRR}\cite{SmithCaiTasissaRieEDMC2}\cite{GhoshTasissaIRLSEDMC}\cite{CriscitielloSNLFullLandScape}: (i) non-degenerate trilateration graphs\cite{NetworkLoc1} are known to be universally rigid\cite[Coro. 1]{ZhuSoUniverRigid}, yet this relationship is nowhere inclusive. Moreover, to the best of our knowledge, there are no effective algorithms to test the universal rigidity except solving the Biswas-Ye Semidefinite Relaxation (SDR)\cite{SoAnthonySNL2006}; (ii) Semidefinite Programming (SDP) based approaches convexify the matrix dimension to $n\times n$, causing extra computational burden when $n\gg r$. Skeptical readers may argue that \eqref{eq_SDR_EDMC_prob} can be tackled by convex-lifting (over-parameterization) to meet the SDR rank lower bound\cite{AnthonyUnifRankRedSDR}, combined with sophisticated manifold optimization strategies, e.g.,\cite[Sec. 4.4]{TangTohRieRankRedSNL}. Yet the success of such routine does not explain practical observations: (i) usually small over-parameterization suffices for non-convex algorithms to evade saddle points\cite[Sec. III-D]{YCLiSunSNLRR} and local minima\cite[Sec. 7]{CriscitielloSNLFullLandScape}; (ii) the recommended numerical method to obtain a successful recovery remains low-rank inductive regularization\cite[Sec. 3.1]{LeungTohSDPDCMP}\cite[Sec. IV]{TasissaEDMCProof}. Therefore, we are interested in answering the following basic question\footnote{Please notice that low-rankness alone can not fully characterize an EDM\cite{DokmanicEDMTheory}, and direct adaptation of Burer-Monteiro factorization or nuclear norm minimization to an EDM without leveraging rank-one quadratic sensing structure w.r.t. $\boldsymbol{\omega}_{\boldsymbol{\alpha}}$ basis in \eqref{eq_SDR_EDMC_prob} leads to poor performance\cite[Sec. I-B]{YCLiSunSNLRR}.}

\textbf{Q1:} \textit{Does the theoretical success of the classical two-step non-convex factorization type incoherent Low-rank Matrix Completion (LRMC) algorithms \textnormal{\cite{KeshavanOptSpace}\cite{SL15NonCVXFR}\cite{ZhengLaffertyNonCVXFR}\cite{ChenWainLREstPGD}} generalize to the EDMC setting?}

As depicted in our previous work\cite{YCLiSunSNLRR}, a twofold dilemma arises when addressing \textbf{Q1}. \textbf{First}, practical sampling schemes in real-world applications of the EDMC problem often deviate from the ideal Bernoulli model. \textbf{Second}, the correlation matrix of the $\boldsymbol{\omega}_{\boldsymbol{\alpha}}$ basis (denoted as $\mathbf{H}_{\boldsymbol{\omega}}$, c.f. \eqref{eq_def_of_correlaMat_Hw}) is non-diagonal, causing simple re-weighting scheme as in Hankel Matrix Completion (HMC)\cite{ChenChiSpectCompreHMC}\cite{CaiWangSpecCS_PGD} no longer valid. This work addresses the second challenge by providing the first non-convex routine to solve EDMC with global convergence guarantees, without requiring sample splitting, and our approach leverages different proof techniques compared to\cite{YCLiSunSNLRR}\cite{SmithCaiTasissaRieEDMC2}\cite{GhoshTasissaIRLSEDMC}\cite{SmithRieOptEDMCSubGChaos}. We now discuss these and highlight the difference in greater detail.
\subsection{Related Work}
\label{subsec_relatedWork}
While the Bernoulli sample model used in standard LRMC tasks does not have strong real-world application scenarios when specified to EDMC, we are not aware of any work that analyzed the non-convex gradient refinement stage under a more practical sample model that depends on the magnitudes of EDM entities, i.e., the so-called unit ball rule, under the random geometry setup\footnote{Transition bound obtained from random geometry graph is often sharp up to constant, yet it also requires ideal distribution assumptions on the point set.}. Javanmard and Montanari developed the stability notion of a rigid graph\cite[Lemma 4.1, 4.2]{JavanmardMontanariTrEDMC} and use that to show the perturbation on convex solution can be bounded. \cite{KarbasiOhMAPMDSTri}\cite{OhMontanariMDS_MAP} analyzed the MAP-MDS estimator\cite{ShangYDistMDS} by showing the shortest path algorithm faithfully returns the Euclidean distance between unconnected nodes when the point set is regular. In parallel, \cite{MontanariOhSEDMCOptSpace} obtained almost the same result, and they sent the surrogate produced by MAP-MDS into the OptSpace\cite{KeshavanOptSpace} for refinement. Numerically, it was demonstrated that the phase transition of OptSpace matches that of the theoretical bound on MAP-MDS. \cite[Thm. 1.2]{ParhizkarOptSpace} provided bounds on similar refinement step, yet they treat the EDM-dependent samples as pure additive noise and only address the Bernoulli erasures.

The original motivation for raising \textbf{Q1} comes from Tasissa and Lai\cite{TasissaEDMCProof}, who showed, for the first time that a natural parallelization of the Nuclear Norm Minimization (NNM)\cite{LRMC_Can1}\cite{RechtFazelNNMSurvey}, called trace minimization (the trace regularized version of the SDR of \eqref{eq_SDR_EDMC_prob}), can solve EDMC given $\mathcal{O}(n\nu r\log^2 n)$ uniform distance samples taking at random, where $\nu$ is certain coherence parameter defined w.r.t. the EDM basis $\boldsymbol{\omega}_{\boldsymbol{\alpha}}$\footnote{In general $\nu=\mathcal{O}(r\mu)$, where $\mu$ is the standard coherence parameter defined w.r.t. $\mathbf{e}_i\mathbf{e}_j^T$, c.f., \cite[Def. 1.2]{LRMC_Can1}.}. They also demonstrated the generalizability of this specific modification to the Golfing Scheme\cite{GrossLRMCProof} by analyzing matrix completion problem under an arbitrary set of non-orthonormal and non-sub-Gaussian basis, provided the extremal eigenvalues of the basis correlation matrix are well understood\cite{TasissaLRMC_genebasis}. Subsequent work \cite{LichtenbergTasissaEDMC} characterized the eigen-spectrum of the EDM basis correlation matrix $\mathbf{H}_{\boldsymbol{\omega}}$, which is defined as
\begin{gather}
	\setlength\belowdisplayskip{3pt}
	\setlength\abovedisplayskip{3pt}
	\label{eq_def_of_correlaMat_Hw}
	\mathbf{H}_{\boldsymbol{\omega}}:=\mathbf{WW}^T\in\mathbb{R}^{L\times L},\,L:=n(n-1)/2=|\mathbb{I}|,\\
	\mathbf{W}:=[\mathrm{vec}(\boldsymbol{\omega}_{(1,2)}),...,\mathrm{vec}(\boldsymbol{\omega}_{(n-1,n)})]^T\in\mathbb{R}^{L\times n^2},\nonumber
\end{gather}
where $\mathrm{vec}(\cdot)$ stands for matrix vectorization. \cite[Coro. 2.2]{LichtenbergTasissaEDMC} revealed that $\mathbf{H}_{\boldsymbol{\omega}}$ has only three types of distinct eigenvalues: $2n,\,n$, and $2$. Such near-singular spectral structure does not significantly hinder the convergence of the Golfing Scheme, as the original inexact dual certificate requires precision on $\mathcal{O}(1/n)$ for the residual term\cite[Prop. 2-(2a)]{ChenIncoOptimalMC}, while the modified one for EDMC appears like $\mathcal{O}(\frac{\lambda_{\min}(\mathbf{H}_{\boldsymbol{\omega}})}{n\lambda_{\max}(\mathbf{H}_{\boldsymbol{\omega}})})=\mathcal{O}(1/n^2)$\cite[Thm. 2]{TasissaLRMC_genebasis}. Since the Golfing Scheme converges exponentially fast, the extra iteration cost for improving $1/n$ to $1/n^2$ is negligible. Yet this makes the design and analysis of non-convex methods considerably more obscure, which, in its most straightforward form, aims to solve the s-stress cost function defined below
\begin{align}
	\setlength\belowdisplayskip{3pt}
	\setlength\abovedisplayskip{3pt}
	\label{eq_s_stress_forthorder_form}
	\min_{\mathbf{P}\in\mathbb{R}^{n\times r}}\,h(\mathbf{P}):&=\sum_{\boldsymbol{\alpha}\in\Omega} |\langle\mathbf{PP}^T,\boldsymbol{\omega}_{\boldsymbol{\alpha}}\rangle - \mathbf{D}^{\star}_{\boldsymbol{\alpha}}|^2,\\
	&=\Vert\mathcal{P}_{\Omega}(g(\mathbf{PP}^T)-\mathbf{D}^{\star})\Vert_F^2,\nonumber\\
	[\mathcal{P}_{\Omega}(\mathbf{D})]_{ij}&=
	\begin{cases}
		\mathbf{D}_{ij},\,\,&\text{if}\,\,(i,j)\in \Omega\\
		0,\,\,&\text{else}
	\end{cases}.\nonumber
\end{align}
Where standard notation $\mathcal{P}_{\Omega}$ is adopted to define the projection onto sampled entities, and linear operator $g:\mathcal{S}(n)\to \mathcal{S}(n)$ maps Gram matrix $\mathbf{G}=\mathbf{PP}^T$ to the EDM generated by $\mathbf{P}$. Here, $\mathcal{S}(n)$ denotes the set of symmetric matrix in $\mathbb{R}^{n\times n}$, see Section \ref{subsec_math_setup} for formal definitions of $g$.
Eq. \eqref{eq_s_stress_forthorder_form} bears great resemblance to the Wirtinger Flow\cite{PhaseretrievalWirtinger} used in phase retrieval, Hankel lift for spectral compressive sensing\cite{CaiWangSpecCS_PGD}\cite{LiCuiPGD_SpectralCSHankel}, $l_{2,\infty}$ regularized free non-convex matrix completion\cite{MaImplicitRegularNonCVX_GD}\cite{ChenDLiLOO_RecRegular_Mat}, and so on.
However, the large condition number of $\mathbf{H}_{\boldsymbol{\omega}}$ prevents one from pursuing theoretical contraction speed when optimizing \eqref{eq_s_stress_forthorder_form}, as discussed in \cite[Sec. III-C]{YCLiSunSNLRR}\cite[Sec. 6]{SmithCaiTasissaRieEDMC2}, even though \eqref{eq_s_stress_forthorder_form} does produce an attractive basin given enough Bernoulli samples, the basin is too small when measured by $l_{2,\infty}$ norm. 

Such predicament stems from the strong diagonal dominance inherent in the $\boldsymbol{\omega}_{\boldsymbol{\alpha}}$ basis. A total sum of them gives rise to $\boldsymbol{\mathcal{L}}=n\mathbf{I}-\mathbf{11}^T$, where $\mathbf{I}$ is the identity matrix. While one can show the sampled sum closely approximates $\boldsymbol{\mathcal{L}}$, controlling the distortion introduced by the $n\mathbf{I}$ term remains challenging\cite[Lemma B.2]{YCLiSunSNLRR}. Nevertheless, extensive numerical evidence over the years has proved that various first-order methods perform well when optimizing \eqref{eq_s_stress_forthorder_form} without relying on explicit regularization or incoherence projections. These approaches not only converge when started randomly or by simple spectral estimator under Bernoulli sample model\cite{YCLiSunEDMC_Spec_init}, but such local search is also acknowledged to be remarkably effective in the refining stage of several SDRs\cite[Sec. 5]{BiswasSDRSNLacm}\cite[Sec. 3.2]{LeungTohSDPDCMP}, even when the sampling model is substantially more intricate. We pause to emphasize that under Bernoulli rule, the original two-step routine, i.e., spectral initialization\cite{ZhangSVD_MDS}\cite[Lemma 5.6]{SmithCaiTasissaRieEDMC2}\cite[Thm. I.1]{YCLiSunEDMC_Spec_init} followed by Vanilla Gradient Descent (VGD) on \eqref{eq_s_stress_forthorder_form}, fails to match the empirical performance of trace minimization-based approaches. Such unexpected discrepancy has also been witnessed in\cite[Sec. 5]{GhoshTasissaIRLSEDMC}, where they showed the phase transition of a scaled-stochastic gradient descent\cite{ZhangChiuScaleSGD} is evidently later than that of convex and the iterative re-weighting method. 
\begin{table*}[t]
	\renewcommand\arraystretch{1.2}
	\centering
	\caption{Comparison with recent works based on dual basis expansion}
	\label{comparation_of_recent_works}
	\scalebox{0.95}{\begin{tabular}{cccccc}
			\hline
			Algorithm Type & Incoherence Type & Claimed Sample Complexity & Global Recovery & Sample Splitting & Riemannian Geometry \\ \hline
			Trace Minimization\cite{TasissaEDMCProof}  & Strong           &  $\mathcal{O}(n r\log^2 n)$\cite[Thm. 5]{TasissaEDMCProof}                 & \checkmark             & N.A.     & N.A.        \\
			Iterative Fast Hard Thresholding\cite{SmithCaiTasissaRieEDMC2}      & Strong         &  $\mathcal{O}(nr^2 \log^2 n)$\cite[Thm. 5.9]{SmithCaiTasissaRieEDMC2}             & \checkmark      & \checkmark & Embedding           \\
			Iterative Re-weighting Least Square\cite{GhoshTasissaIRLSEDMC}     & Strong           &    $\mathcal{O}(n r\log n)$\cite[Thm. 4.3]{GhoshTasissaIRLSEDMC}       & $\times$               & $\times$        & N.A.       \\
			Projected Gradient Descent \textbf{(This work)} & Standard         &   $\mathcal{O}(nr^3\log n)$       & \checkmark              & $\times$      & Quotient         \\ \hline
	\end{tabular}}\vspace{-12pt}
\end{table*}

Another long-standing yet recently advanced open question concerning \eqref{eq_s_stress_forthorder_form} is the analytical justification for the observed absence of its spurious local minima, as summarized and empirically reported in\cite[Ch. 3.5]{ParhizkarPhDthsis}. When the observation is complete and noiseless, this conjecture was recently disproven by Criscitiello et al.\cite[Sec. 6]{CriscitielloSNLFullLandScape}. They construct explicit spurious configurations and provide sufficient condition for s-stress to admit strict second-order local minima.
An interesting yet subtle connection between their work and ours lies in\cite[Thm. 5.1]{CriscitielloSNLFullLandScape}. Loosely speaking, they show that if  
\begin{equation}
	\setlength\belowdisplayskip{3pt}
	\setlength\abovedisplayskip{3pt}
	\label{CriscitielloCondForBenignLandscape}
	(k+2)\sigma_{r}^{\star}>4n\max_{i\in [n]}\Vert\mathbf{e}_i^T\mathbf{P}^{\star}\Vert_2^2, \,\Omega =\mathbb{I},
\end{equation}
holds, then any over-parameterized second-order critical $\mathbf{P}\in\mathbb{R}^{n\times k}$, with $k>r$, is a global optimum of \eqref{eq_s_stress_forthorder_form}, where $\sigma_{r}^{\star}$ is the smallest singular value of $\mathbf{G}^{\star}$. Though sharing remarkable resemblance with the incoherence assumption\cite[Def. 1.2]{LRMC_Can1}, it is straightforward to check \eqref{CriscitielloCondForBenignLandscape} can only hold in the regime $k>r$.
And as we will elaborate further, it also remains an open question whether the standard incoherence condition adequately captures the class of EDMs that are efficiently completable in polynomial time\footnote{Please see \cite[Thm. 4]{KirayCoherenceRec} for more discussion.}.

Given that many of these largely unexplored questions stem, at least in part, from the condition number of $\mathbf{H}_{\boldsymbol{\omega}}$, \cite{SmithCaiTasissaRieEDMC2} considered using the pseudo inverse of $g$ (denoted as $g^+$) to correct the condition number of the Hessian of \eqref{eq_s_stress_forthorder_form}. That is, to construct the following new pre-conditioned sample operator
\begin{equation}
	\setlength\belowdisplayskip{3pt}
	\setlength\abovedisplayskip{3pt}
	\label{eq_R_Omega_sample_first_time}
	\mathcal{R}_{\Omega}(\cdot):=g^+\mathcal{P}_{\Omega}g(\cdot),
\end{equation}
and to optimize 
\begin{gather}
	\setlength\belowdisplayskip{3pt}
	\setlength\abovedisplayskip{3pt}
	\label{eq_RDG_EDMC_Cost}
	\min_{\mathbf{G}\in\mathcal{S}^{n,r}_+} F(\mathbf{G}):=\frac{1}{p}\langle\mathcal{R}_{\Omega}(\mathbf{G}-\mathbf{G}^{\star}),\mathbf{G}-\mathbf{G}^{\star}\rangle,\\
	\nabla F(\mathbf{G}) = \frac{1}{p}\mathcal{R}_{\Omega}(\mathbf{G}-\mathbf{G}^{\star}) + \frac{1}{p}\mathcal{R}_{\Omega}^{*}(\mathbf{G}-\mathbf{G}^{\star}),\nonumber
\end{gather}
where $\mathcal{S}^{n,r}_+$ denotes the Positive Semidefinite (PSD) rank-$r$ embedding matrix manifold\cite{VandereyckenEmbgeoSnp}. The major drawback is that one cannot compute $\nabla F(\mathbf{G})$ directly, but only its first term, which they called pseudo gradient (after proper rescaling) $\hat{\nabla}F(\mathbf{G}):=\frac{2}{p}\mathcal{R}_{\Omega}(\mathbf{G}-\mathbf{G}^{\star})$, since the second term requires precise knowledge of $\mathbf{G}^{\star}$ but not just sampled distances. Certain trade-off arises here between maintaining self-adjointness and ensuring algorithmic explainability: since $\mathbf{H}_{\boldsymbol{\omega}}^{-1}$ (the inverse of $\mathbf{H}_{\boldsymbol{\omega}}$) is not diagonal, in general we would like to evade from constructing $\mathcal{R}_{\Omega}^*\mathcal{R}_{\Omega}$\footnote{Finding uniform bounds on $\mathcal{R}_{\Omega}^*\mathcal{R}_{\Omega}$, e.g., analogy of Lemma \ref{lema_Romega_RIPL}, \ref{eq_NromalSpace_tight_bound} is challenging, we will come back to this with greater detail in Section \ref{subsec_whyAPGD}.}, while using $\mathcal{R}_{\Omega}$ alone compromises the least-squares structure. The situation becomes more subtle when one inquires about numerical performance. This Riemannian pseudo gradient descent does not exhibit strong convergence in the sample-limited regime when compared to \eqref{eq_s_stress_forthorder_form}, despite its theoretical error contraction rate resembling that of the classic Riemannian gradient descent for LRMC\cite{WeiRCGLRMC_proof}.
\subsection{Major Contributions and Organization}
\label{subsec_MajorCon}
By introducing Burer-Monteiro factorization into \eqref{eq_RDG_EDMC_Cost}, and recalculating the pseudo gradient, we have
\begin{subequations}
	\label{eq_fracted_pre_cond_cost_funcand_Grad}
\begin{gather}
	\setlength\belowdisplayskip{3pt}
	\setlength\abovedisplayskip{3pt}
	\label{eq_fracted_pre_cond_cost_func}
	\min_{\mathbf{P}\in\mathbb{R}^{n\times r}} f(\mathbf{P}) = \frac{1}{p} \langle\mathcal{R}_{\Omega}(\mathbf{P}\mathbf{P}^T-\mathbf{G}^{\star}),\mathbf{P}\mathbf{P}^T-\mathbf{G}^{\star}\rangle,\\
	\label{eq_fracted_pre_cond_cost_Grad}
	\hat{\nabla} f(\mathbf{P}) = \frac{2}{p}\mathcal{R}_{\Omega}(\mathbf{P}\mathbf{P}^T-\mathbf{G}^{\star})\mathbf{P},
\end{gather}	
\end{subequations}
which, can be viewed as optimizing $\mathbf{P}$ over a quotient manifold. Compared with \cite{YCLiSunEDMC_Spec_init}, our major contributions are twofold:
\begin{itemize}
	\item We analyze the contraction behavior when using \eqref{eq_fracted_pre_cond_cost_Grad} to perform Projected Gradient Descent (PGD), starting from a surrogate returned by a recently proposed spectral estimator\cite{YCLiSunEDMC_Spec_init}\cite[Lemma 5.6]{SmithCaiTasissaRieEDMC2}, called One-step MDS (OS-MDS). This routine is referred to as the Asymmetric Projected Gradient Descent (APGD), where the ``projection" step enforces iteration to stay incoherent, similar to the setup in\cite{ZhengLaffertyNonCVXFR}. And the name ``asymmetric" comes from the fact that linear sensing operator $\mathcal{R}_{\Omega}$ is not self-adjoint, necessitating a delicate analysis of different components of the iteration residual, a feature not typically encountered when dealing with restricted strong convexity w.r.t. the first-order derivative. We show a near-linear convergence to the ground truth of the APGD iteration can be established given $\mathcal{O}(\mu^2 r^3 \kappa^2 n \log n)$ random distance samples by constructing standard regularity condition w.r.t. the pseudo gradient direction. 
	\item We provide two new uniform estimates for controlling two types of the non-tangent space residuals during iteration in inner product form, namely, Lemma \ref{lema_sharp_bound_for_L21L22} and \ref{lema_nonUni_bound_for_L3}, which act as analogies and extensions of the random graph lemma used in classic non-convex LRMC tasks\cite[Lemma 7.1]{KeshavanOptSpace}\cite[Lemma 5]{CaiWangSpecCS_PGD}. These bounds are developed based on a celebrated sharp result by Bandeira and van Handel\cite{BandeiraHandelSharpRGL}, as well as the separation trick originated from Bhojanapalli and Jain\cite{pmlr-v32-bhojanapalli14} in deterministic matrix completion context with Ramanujan graphs, substantially different from the one used in our previous work\cite[Lemma B.2]{YCLiSunSNLRR}. Moreover, these two estimates match the best known bound in incoherent LRMC for resemblance purposes\cite[Lemma 8]{ChenJiLiModelFreeMC} up to $\log n$ factor\footnote{Readers interested in comparing these non-tangent space concentration results are referred to \cite[Sec. 4.1]{ChenJiLiModelFreeMC}.}. Rationales behind our choice of analyzing local behavior of \eqref{eq_fracted_pre_cond_cost_Grad}, and the current dilemma between analytical tractability and algorithmic fidelity behind this specific attempt, are explained in Section \ref{subsec_whyAPGD}.
\end{itemize}
We also draw numerical comparisons between the OS-MDS initialized unregularized s-stress \eqref{eq_s_stress_forthorder_form}, IFHT-EDMC\cite[Alg. 2]{SmithCaiTasissaRieEDMC2}, and APGD with vanilla Barzilai-Borwein (BB) stepsize over synthetic point set. While the APGD is poorly behaved and appears to possess practical significance only in the near-asymptotic region, phenomena analogous to the role of incoherence restrictions in non-convex LRMC\cite[Ch. 2]{sun2015matrix} are observed. This raises interesting questions regarding the design principles of non-convex matrix recovery strategies under non-orthonormal, non-sub-Gaussian, and ill-conditioned sample basis -- is EDMC a special case? 

This section concludes with a comparison of the theoretical results presented in this work with those from three recent studies\cite{TasissaEDMCProof}\cite{SmithCaiTasissaRieEDMC2}\cite{GhoshTasissaIRLSEDMC}, as summarized by Table \ref{comparation_of_recent_works}. While similarities lie in the assumption of incoherence restrictions, uniform sampling, and the use of dual basis expansion, their algorithmic frameworks vary. The term ``strong incoherence" is referred to either\cite[Eq. (3)]{ChenIncoOptimalMC} or\cite[Def. 1]{TasissaEDMCProof}, while the former being known to be unnecessary in LRMC\footnote{We omit extra $\mathcal{O}(\mathrm{poly}(r))$ factor that may implicitly contained by the strong incoherence parameter when drawing the third column of Table \ref{comparation_of_recent_works}.}, reducing the latter in modified Golfing scheme for EDMC remains an open question. 

\textit{Organization: } In Section \ref{subsec_math_setup} we provide a more detailed background and model setup. We state our main theorem and two important lemmas that control the initialization error of OS-MDS and contraction behavior of APGD in Section \ref{subsec_main_result}. The proofs are deferred to the Appendix. Section \ref{sec_numericalRes_Disscuss} contains the aforementioned numerical result and the manuscript is concluded with a discussion in Section \ref{sec_conclude_Diss}.

\textit{Notation: } $\mathbf{y}$, $\mathbf{Y}$ stands for column vectors and matrices.
$\mathbf{Y}_{ij}$, $\mathrm{diag}(\mathbf{Y})$, and $\mathrm{diag}(\mathbf{y})$ is the $(i,j)$-th entity, column vector formed by the diagonal elements of $\mathbf{Y}$, and diagonal matrix formed by $\mathbf{y}$, respectively. $\mathbf{Y}_{\cdot,k}$ / $\mathbf{Y}_{k,\cdot}^T$ stands for the $k$-th column/row of $\mathbf{Y}$. $\mathrm{Od}(\mathbf{Y})$ means a matrix obtained by nullifying the diagonal of $\mathbf{Y}$. $\mathcal{F}^{*}$ stands for the adjoint of a linear operator $\mathcal{F}$, and $\mathcal{I}$ denotes the identity operator. $\circ$ is the Hadamard product. $\Vert\mathbf{Y}\Vert$, $\Vert\mathbf{Y}\Vert_F$, $\Vert\mathbf{Y}\Vert_{\infty}$, $\Vert\mathbf{Y}\Vert_{2,\infty}$, $\Vert\mathbf{Y}\Vert_{*}$ stand for the spectral, Frobenius, entry-wise $l_{\infty}$, row-wise $l_{2,\infty}$, and nuclear norm of $\mathbf{Y}$. $O(r)$ stands for the $r$ dimensional orthogonal group. $\Vert\mathbf{y}\Vert_2$, $\Vert\mathbf{y}\Vert_{\infty}$ is the vector $l_2$, $l_\infty$ norm. $C,\,c>0$ are universal constants that may differ from line to line. We retain $\beta>2$ for expression like $n^{1-\beta}$, and $C_{\beta}$ denotes a constant that dependent with $\beta$ and some $c_I>4$. $\epsilon$ and $\varepsilon\leq1$ are used to denote absolute numbers that control finite sample error. $a\lesssim b$ and $a\gtrsim b$ stands for $a<Cb$ and $a>Cb$, respectively. $\mathbf{Y}\succcurlyeq\mathbf{0}$ means that $\mathbf{Y}$ is PSD. $\mathbb{S}^{n-1}$ stands for the unit sphere in $\mathbb{R}^n$ centered at origin.
\section{Mathematical Setup and Main Result}
\label{subsec_math_setup}

Without loss of generality, we can always assume $\mathbf{P}^{\star T}\mathbf{1}=\mathbf{0}$, and all considered point sets are in the centered subspace $\mathcal{S}_c^n:=\{\mathbf{A}:\mathbf{A1}=\mathbf{0},\mathbf{A}\in\mathcal{S}(n)\}$ to remove the translation ambiguity. For notation simplicity, we abuse $\mathcal{I}$ to denote its $\mathcal{S}_c^n$ restriction $\mathcal{I}_{\{\mathbf{1}\}^{\perp}}$, i.e., $\mathcal{I}_{\{\mathbf{1}\}^{\perp}}(\mathbf{A})=\mathbf{A}$ if $\mathbf{A}\in \mathcal{S}_c^n$. One may also use different geometry centers\cite{LichtenberLocDisMat}. The rotation and reflection ambiguity in EDMC gives rise to a group synchronization problem over $O(r)$, and we use the solution of the orthogonal Procrustes problem to define the distance between $\mathbf{P}$ and $\mathbf{P}^{\star}$
\begin{equation}
	\setlength\belowdisplayskip{3pt}
	\setlength\abovedisplayskip{3pt}
	\label{eq_goedist_quotient}
	\mathrm{dist}(\mathbf{P},\mathbf{P}^{\star})^2=\Vert\Delta\Vert_F^2:=\min_{\boldsymbol{\psi}\in O(r)}\Vert\mathbf{P}-\mathbf{P}^{\star}\boldsymbol{\psi}\Vert_F^2,
\end{equation}
and the solution of \eqref{eq_goedist_quotient} is denoted as $\boldsymbol{\psi}^{\star}$. It satisfies $\mathbf{P}^T\mathbf{P}^{\star}\boldsymbol{\psi}^{\star}\succcurlyeq \boldsymbol{0}$ and $\Delta^T\mathbf{P}^{\star}\boldsymbol{\psi}^{\star}\in \mathcal{S}(r)$\cite[Lemma 6]{RongGeSpuriousLocalMinima}. We also let $\Delta$ and $\boldsymbol{\psi}^{\star}$ inherit any super/sub-script of $\mathbf{P}$, i.e., $\mathrm{dist}(\mathbf{P}^{(\cdot)},\mathbf{P}^{\star})=\Vert\Delta^{(\cdot)}\Vert_F=\Vert\mathbf{P}^{(\cdot)}-\mathbf{P}^{\star}\boldsymbol{\psi}^{\star(\cdot)}\Vert_F$.

The forward and inverse EDM mapping is needed to explicitly define the sample operator in \eqref{eq_R_Omega_sample_first_time}, we restate them below
\begin{gather}
	\setlength\belowdisplayskip{2pt}
	\setlength\abovedisplayskip{2pt}
	\label{eq_Gram_to_EDM}
	g(\mathbf{G}):=\mathrm{diag}(\mathbf{G})\mathbf{1}^T+\mathbf{1}\mathrm{diag}(\mathbf{G})^T-2\mathbf{G},\\
	\label{eq_EDM_to_Gram}
	g^+(\mathbf{D}):=-\frac{1}{2}\mathbf{J}\mathbf{D}\mathbf{J},\,\mathbf{J}=\mathbf{I}-\frac{1}{n}\mathbf{1}\mathbf{1}^T.
\end{gather}
The ground truth EDM is sampled by the symmetric Bernoulli rule, that is $\Omega:=\{(i,j):\delta_{\boldsymbol{\alpha}}=1, \boldsymbol{\alpha}\in\mathbb{I}\}$, where $\delta_{\boldsymbol{\alpha}}$ is a list of i. i. d. 0/1 Bernoulli random variables with $\mathbb{P}(\delta_{\boldsymbol{\alpha}}=1)=p$. Let $\mathcal{P}_{\Omega}$ denote the symmetric sample operator
\begin{equation}
	\setlength\belowdisplayskip{3pt}
	\setlength\abovedisplayskip{3pt}
	\label{eq_P_omega_model}
	\mathcal{P}_{\Omega}(\mathbf{D}):=\sum_{\boldsymbol{\alpha}\in\mathbb{I}}\delta_{\boldsymbol{\alpha}}(\langle\mathbf{D},\mathbf{e}_i\mathbf{e}_j^T\rangle\mathbf{e}_i\mathbf{e}_j^T+\langle\mathbf{D},\mathbf{e}_j\mathbf{e}_i^T\rangle\mathbf{e}_j\mathbf{e}_i^T).
\end{equation}
A list of works\cite{LichtenbergTasissaEDMC}\cite{LichtenberLocDisMat}\cite{SmithCaiRieOptED1} shows that 
the following primal-dual basis decomposition of $\mathcal{I}_{\{\mathbf{1}\}^{\perp}}$ holds over $\mathcal{S}_c^n$
\begin{align}
	\setlength\belowdisplayskip{-2pt}
	\setlength\abovedisplayskip{-2pt}
	\mathcal{I}_{\{\mathbf{1}\}^{\perp}}(\cdot)&:=g^+g(\cdot)=-\frac{1}{2}\sum_{\boldsymbol{\alpha}\in\mathbb{I}}\langle \cdot,\boldsymbol{\omega}_{\boldsymbol{\alpha}}\rangle\mathbf{J}(\mathbf{e}_i\mathbf{e}_j^T+\mathbf{e}_j\mathbf{e}_i^T)\mathbf{J}\nonumber\\
	\label{eq_the_dual_basis_rep}
	&=\sum_{\boldsymbol{\alpha}\in\mathbb{I}}\langle \cdot,\boldsymbol{\omega}_{\boldsymbol{\alpha}}\rangle\boldsymbol{\nu}_{\boldsymbol{\alpha}},
\end{align}
and they call $\boldsymbol{\nu}_{\boldsymbol{\alpha}}:=-\frac{1}{2}\mathbf{J}(\mathbf{e}_i\mathbf{e}_j^T+\mathbf{e}_j\mathbf{e}_i^T)\mathbf{J}$ the dual-basis of $\boldsymbol{\omega}_{\boldsymbol{\alpha}}$\cite{LichtenbergTasissaEDMC}. From \eqref{eq_the_dual_basis_rep}, we find the explicit expression of the pre-conditioned sample operator $\mathcal{R}_{\Omega}$
\begin{equation}
	\setlength\belowdisplayskip{3pt}
	\setlength\abovedisplayskip{3pt}
	\mathcal{R}_{\Omega}(\cdot):=g^+\mathcal{P}_{\Omega}g(\cdot)=\sum_{\boldsymbol{\alpha}\in\mathbb{I}}\delta_{\boldsymbol{\alpha}}\langle \cdot,\boldsymbol{\omega}_{\boldsymbol{\alpha}}\rangle\boldsymbol{\nu}_{\boldsymbol{\alpha}}.
\end{equation}
Under the symmetric Bernoulli model with parameter $p$, we have $\mathbb{E}\frac{1}{p}\mathcal{R}_{\Omega}=\mathbb{E}\frac{1}{p}\mathcal{R}_{\Omega}^*=\mathcal{I}$. Finally, let $\mathcal{Q}_{\Omega}$ be the sensing operator in \eqref{eq_s_stress_forthorder_form}, that is
\begin{equation}
	\setlength\belowdisplayskip{3pt}
	\setlength\abovedisplayskip{3pt}
	\label{eq_Q_omega_s_stress_sample}
	\mathcal{Q}_{\Omega}(\cdot):=\mathcal{P}_{\Omega}g(\cdot)=\sum_{\boldsymbol{\alpha}\in\mathbb{I}}\delta_{\boldsymbol{\alpha}}\langle \cdot,\boldsymbol{\omega}_{\boldsymbol{\alpha}}\rangle(\mathbf{e}_i\mathbf{e}_j^T+\mathbf{e}_j\mathbf{e}_i^T),
\end{equation}
and we abuse
\begin{equation*}
	\setlength\belowdisplayskip{3pt}
	\setlength\abovedisplayskip{3pt}
	\tilde{\mathcal{Q}}_{\Omega}(\cdot):=\sum_{\boldsymbol{\alpha}\in\mathbb{I}}\delta_{\boldsymbol{\alpha}}\langle\cdot,
	\boldsymbol{\omega}_{\boldsymbol{\alpha}}\rangle\mathbf{e}_{\boldsymbol{\alpha}},\,\tilde{\mathcal{Q}}_{\Omega}:\mathbb{R}^{n\times n}\to \mathbb{R}^L,
\end{equation*} 
to denote the re-scaled vectorization of \eqref{eq_Q_omega_s_stress_sample}, where $\mathbf{e}_{\boldsymbol{\alpha}}$ is the $\boldsymbol{\alpha}$-th canonical Euclidean basis in $\mathbb{R}^L$. It is straightforward to show $\mathcal{Q}_{\Omega}^*\mathcal{Q}_{\Omega}=2\tilde{\mathcal{Q}}_{\Omega}^*\tilde{\mathcal{Q}}_{\Omega}$.

\subsection{Main Result}
\label{subsec_main_result}
We are now ready to state our main result. Assume that $\mathbf{P}^{\star}$ satisfies the standard incoherence condition with $1\leq\mu\leq\frac{n}{r}$, namely, let $\mathbf{P}^{\star}\mathbf{P}^{\star T}=\mathbf{U}^{\star}\boldsymbol{\Sigma}^{\star}\mathbf{U}^{\star T}$ denotes its rank $r$ thin Singular Value Decomposition (SVD), we have
\begin{equation}
	\setlength\belowdisplayskip{3pt}
	\setlength\abovedisplayskip{3pt}
	\label{eq_standard_incoherent_para}
	\Vert\mathbf{U}^{\star}\Vert_{2,\infty}\leq\sqrt{\frac{\mu r}{n}},\Vert\mathbf{P}^{\star}\Vert_{2,\infty}\leq\sqrt{\frac{\mu r\sigma_1^{\star}}{n}},
\end{equation}
where $\sigma_i^{\star}=\boldsymbol{\Sigma}^{\star}_{ii}$, and let $\kappa=\sigma_1^{\star}/\sigma_r^{\star}$. 
The proposed APGD iteration is listed above as Algorithm \ref{alg_APGD_Resample}, where $\mathcal{P}_I(\mathbf{P}_k)$ is defined as the following trimming step to enforce incoherence
\begin{equation}
	\setlength\belowdisplayskip{3pt}
	\setlength\abovedisplayskip{3pt}
	\label{eq_the_trimming_step}
	\mathcal{P}_I(\mathbf{P}_{k}(i,:))=
	\begin{cases}
		\mathbf{P}_{k}(i,:),&\text{if } \Vert\mathbf{P}_{k}(i,:)\Vert_2\leq\sqrt{\frac{\mu r\sigma_1^{\star}}{n}}\\
		\frac{\mathbf{P}_{k}(i,:)}{\Vert\mathbf{P}_{k}(i,:)\Vert_2}\sqrt{\frac{\mu r\sigma_1^{\star}}{n}},& \text{otherwise}
	\end{cases}.
\end{equation}
Where $\mathbf{Y}(i,:)$ stands for the $i$-th row of $\mathbf{Y}$, and $\mathcal{T}_r$ denotes the rank $r$ SVD truncation.
\newtheorem{theorem}{Theorem}[section]
\begin{theorem}
	\label{thm_main_global_contraction}
	Under the above set up, let $p\gtrsim C_{\beta}\frac{\mu^2 r^3 \kappa^2\log n}{n}$. For any $k\geq 1$, i.e., Line 1 to Line 5 of Algorithm \ref{alg_APGD_Resample}, we have
	\begin{equation}
		\setlength\belowdisplayskip{3pt}
		\setlength\abovedisplayskip{3pt}
		\label{eq_first_stage_contraction}
		\Vert\Delta_{k}\Vert_F^2\leq(1-\eta^{\prime})^{k-1}\frac{\sigma_r^{\star}}{c_c},
	\end{equation}
	holds with probability at least $1-cn^{1-\beta}-n^{2-\beta}$.
	Where $0<\eta^{\prime}\leq\frac{1}{336c_I\mu r\kappa^2}$, $c_c=\frac{80}{7}$, and $\Delta_{k}$ refers substituting $\mathbf{P}_k$  into \eqref{eq_goedist_quotient}. To reach $\varsigma$ stationary point, i.e., $\Vert\Delta_{k}\Vert_F^2/\Vert\Delta_{0}\Vert_F^2\leq \varsigma$, APGD needs around $336c_I\mu r\kappa^2\log(\frac{1}{\varsigma})$ iterations.
\end{theorem}
The proof of Theorem \ref{thm_main_global_contraction} is developed based on the following two lemmas, which describe the behavior of APGD inside the incoherence and contraction regions (RIC) $\mathcal{B}$, defined below
\begin{equation*}
	\setlength\belowdisplayskip{3pt}
	\setlength\abovedisplayskip{3pt}
	\mathcal{B}:=\{\mathbf{P}\,:\,\Vert\Delta\Vert_{2,\infty}\leq\sqrt{\frac{c_I\mu r\sigma_1^{\star}}{n}},\Vert\Delta\Vert_F^2\leq \frac{\sigma_r^{\star}}{c_c}\}, \,c_I>4.
\end{equation*}
Where the lower bound on $c_I$ comes from $\Vert\Delta\Vert_{2,\infty}\leq\Vert\mathbf{P}\Vert_{2,\infty}+\Vert\mathbf{P}^{\star}\Vert_{2,\infty}\leq 2\Vert\mathbf{P}^{\star}\Vert_{2,\infty}$ by our parameter setting in the incoherence projection step \eqref{eq_the_trimming_step}.
\newtheorem{lemma}{Lemma}[section]
\begin{lemma}
	\label{eq_OS_MDS_the_basic_case}
	The RIC region $\mathcal{B}$ can be entered (i.e., $k=1$, Line 1 in Algorithm \ref{alg_APGD_Resample}) by using a spectral initialization followed by incoherence projection $\mathcal{P}_I$. That is,
	\begin{equation*}
		\setlength\belowdisplayskip{3pt}
		\setlength\abovedisplayskip{3pt}
		\hat{\mathbf{P}}\hat{\mathbf{P}}^T=\mathcal{T}_r\left[ -\frac{1}{2p} \mathbf{J} (\mathcal{P}_{\Omega}\mathbf{D}^{\star}) \mathbf{J}\right],\,\mathbf{P}_1 = \mathcal{P}_{I}(\hat{\mathbf{P}}),
	\end{equation*}
	satisfies $\mathbf{P}_1\in\mathcal{B}$ with probability at least $1-n^{1-\beta}$, provided with $p\gtrsim C_{\beta}\frac{\mu^2 r^3 \kappa^2\log n}{n}$.
\end{lemma}
\begin{proof}
	Please see Appendix \ref{eq_proof_of_OS_mds}.
\end{proof}
\begin{algorithm}[!t]
	\caption{APGD}
	\label{alg_APGD_Resample}
	\begin{algorithmic}[1]
		\REQUIRE Pseudo-gradient $\hat{\nabla}f(\mathbf{P})$ as in \eqref{eq_fracted_pre_cond_cost_Grad}, sampled distances $\mathcal{P}_{\Omega}\mathbf{D}^{\star}$.
		\STATE Initial APGD by OS-MDS as in \cite{YCLiSunEDMC_Spec_init}\cite{SmithCaiTasissaRieEDMC2}
		\begin{equation*}
			\hat{\mathbf{P}}\hat{\mathbf{P}}^T=\mathcal{T}_r\left[ -\frac{1}{2p} \mathbf{J} (\mathcal{P}_{\Omega}\mathbf{D}^{\star}) \mathbf{J}\right],\,\mathbf{P}_1=\mathcal{P}_I(\hat{\mathbf{P}}),
		\end{equation*}
		where $\mathcal{P}_I$ is defined by \eqref{eq_the_trimming_step}.
		\FOR{$k=1,2,\dots$}
		\STATE $\mathbf{P}_{k+1} = \mathbf{P}_k-\frac{2\eta}{p}g^+\mathcal{P}_{\Omega}(g(\mathbf{P}_{k}\mathbf{P}_{k}^T)-\mathbf{D}^{\star})\mathbf{P}_{k}$.
		\STATE $\mathbf{P}_{k+1}=\mathcal{P}_I(\mathbf{P}_{k+1})$.
		\ENDFOR
		\RETURN $\mathbf{P}_k$
	\end{algorithmic}
\end{algorithm}
\begin{lemma}
	\label{thm_main_local_contraction}
	Under the set up of Theorem \ref{thm_main_global_contraction}, inside the RIC $\mathcal{B}$, we have the following update rule for $k\geq 1$ (i.e. Line 2 to Line 5 in Algorithm \ref{alg_APGD_Resample})
	\begin{equation*}
		\setlength\belowdisplayskip{3pt}
		\setlength\abovedisplayskip{3pt}
		\mathbf{P}_{k+1}=\mathcal{P}_I(\mathbf{P}_{k}-\frac{2\eta}{p}\mathcal{R}_{\Omega}(\mathbf{P}_{k}\mathbf{P}_{k}^T-\mathbf{G}^{\star})\mathbf{P}_{k}),
	\end{equation*}
	satisfies $\Vert\Delta_{k+1}\Vert_F^2\leq(1-\frac{\eta\sigma_r^{\star}}{2})\Vert\Delta_k\Vert_F^2$ with probability at least $1-cn^{1-\beta}-n^{2-\beta}$, provided with $p\gtrsim C_{\beta}\frac{\mu^2r^2\kappa^2 \log n}{n}$, and $0<\eta\leq\frac{1}{168\sigma_1^{\star}c_I\mu r\kappa}$.
\end{lemma}
\begin{proof}
	Please see Appendix \ref{AppdiAA_local_contraction}.
\end{proof}
\subsection{Proof of Theorem \ref{thm_main_global_contraction}}
By our setting of parameters, \eqref{eq_first_stage_contraction} is a direct corollary of Lemma \ref{eq_OS_MDS_the_basic_case} and Lemma \ref{thm_main_local_contraction}. The interval on $\eta^{\prime}$ comes from calculating $\frac{\eta\sigma_r^{\star}}{2}$ with the step size $\eta$ determined in Appendix \ref{AppdiAA_local_contraction}. What remains is to count iteration complexity using \eqref{eq_first_stage_contraction}. Suppose choosing the largest possible $\eta^{\prime}=\frac{1}{336c_I\mu r\kappa^2}$, we need to show $(1-\frac{1}{336c_I\mu r\kappa^2})^{k_0-1}\leq \varsigma$, it requires $k_0 \geq \mathcal{O}(336c_I\mu r\kappa^2\log(1/\varsigma))$, concluding the proof.\qed
\newtheorem{remark}{Remark}[section]
\begin{remark}
	\label{remark_class_lrmc_comare}
	It is worth noting that the requirement on $p$ in Lemma \ref{thm_main_local_contraction}, i.e., the establishment of regularity condition for the pseudo gradient, is of the same order in $\kappa$, $\mu$, $r$, and $n$ as its counterpart in classic LRMC theory\cite[Lemma 3]{ZhengLaffertyNonCVXFR}. This observation leads us to conjecture that our techniques might be applied to tighten the estimate by Tasissa and Lai\cite{TasissaEDMCProof}. Meanwhile, the sample complexity for the OS-MDS is slightly suboptimal in $r$ compared with the best-known bound\cite[Thm. 3.23]{ChenSpectralMethods}. The direct rank $r+2$ truncation of an EDM (i.e., the SVD-MDS) has been shown to be minimax sub-optimal in \cite[Sec. IV]{ZhangSVD_MDS}. Since the OS-MDS shares almost same empirical phase transition edge as the SVD-MDS\cite{YCLiSunEDMC_Spec_init}, it suggests that the idea of using $\Vert\frac{1}{p}\mathcal{P}_{\Omega}\mathbf{D}^{\star}-\mathbf{D}^{\star}\Vert$ to construct spectral estimator might be overly conservative in practice.
\end{remark}
\begin{remark}
	Similar to classic results mentioned in Remark \ref{remark_class_lrmc_comare}, step size rule in Lemma \ref{thm_main_local_contraction} depends on both sample complexity and the radius of RIC. We do not opt to optimize $c_c$; instead, $c_I$ is explicitly left as a tunable constant. Enlarging $c_I$ makes the convergence of APGD more fragile, consistent with numerical observations in Section \ref{subsecPT_under_Gaussian_point}. Skeptical reader may have pointed out that step size used in Lemma \ref{thm_main_local_contraction} appears overly conservative. In pre-conditioning free LRMC and HMC contexts, more aggressive step sizes of the order $\mathcal{O}(1/\kappa^2)$\cite{ChenDLiLOO_RecRegular_Mat}, $\mathcal{O}(1/\kappa)$\cite{LiCuiPGD_SpectralCSHankel}\cite{CaiCaiYouStrGDHMC} have been proven feasible. The $\mathcal{O}(1/\kappa)$ rate upon $\mathcal{O}(1/\kappa^2)$ is obtained by reproducing higher order term of $\Vert\Delta\Vert_F$ when establishing regularity condition\cite[Lemma 16, 17]{CaiCaiYouStrGDHMC}. However, the $\mu$, $r$ dependency here can not be trivially moved. The largest component in restricted smoothness part lies in \eqref{eq_upperbound_R_4}, where we utilize Lemma \ref{eq_NromalSpace_tight_bound} to bound $|R_4|$, thus additional $cn\Vert\Delta\Vert_{2,\infty}^2$ term appears. We infer that Lemma \ref{eq_NromalSpace_tight_bound} is not sharp up to $\sqrt{n}$, but we cannot amend it by assuming standard incoherence alone. If this estimate can be shown to match the random graph lemma\cite[Lemma 7.1]{KeshavanOptSpace}, it would immediately yield the local convergence mechanism, i.e., the contraction speed measured by $\Vert \Delta \Vert_F$, when optimizing the s-stress function.
\end{remark}
\subsection{Why APGD and Why not $\mathcal{R}_{\Omega}^*\mathcal{R}_{\Omega}$}
\label{subsec_whyAPGD}
Due to the projection property of $\mathcal{P}_{\Omega}$ under Bernoulli sampling model, it is clear that $\mathcal{P}_{\Omega}^*\mathcal{P}_{\Omega} = \mathcal{P}_{\Omega}$. Obviously, the orthonormality of the sensing basis $\mathbf{e}_i\mathbf{e}_j^T$ ensures this equivalence. However, for the EDMC problem, the primal basis correlation matrix $\mathbf{H}_{\boldsymbol{\omega}}$ and its inverse $\mathbf{H}_{\boldsymbol{\omega}}^{-1}$ (the dual basis correlation matrix\cite{LichtenbergTasissaEDMC}) are not diagonal. Therefore, theoretical analysis of optimizing $\frac{1}{p^2}\Vert\mathcal{R}_{\Omega}(\mathbf{PP}^T-\mathbf{G}^{\star})\Vert_F^2$ and the local behavior of its corresponding gradient, as suggested by\cite{SmithCaiRieOptED1}, should be considerably more obscure. It has been pointed out in\cite{YCLiSunEDMC_Spec_init}\cite[Lemma B.1]{SmithRieOptEDMCSubGChaos} that $\frac{1}{p^2}\mathcal{R}_{\Omega}^*\mathcal{R}_{\Omega}$ is a biased estimator for $\mathcal{I}$, since
\begin{align}
	\setlength\belowdisplayskip{2pt}
	\setlength\abovedisplayskip{2pt}
	&\frac{1}{p^2}\mathcal{R}_{\Omega}^*\mathcal{R}_{\Omega}(\cdot) = \sum_{\boldsymbol{\alpha}
	\neq\boldsymbol{\beta}}\frac{1}{p}\delta_{\boldsymbol{\alpha}}\frac{1}{p}\delta_{\boldsymbol{\beta}}\langle\cdot,\boldsymbol{\omega}_{\boldsymbol{\alpha}}\rangle\langle\boldsymbol{\nu}_{\boldsymbol{\alpha}},\boldsymbol{\nu}_{\boldsymbol{\beta}}\rangle\boldsymbol{\omega}_{\boldsymbol{\beta}}\nonumber\\
	\label{eq_RomegaConjRomega}
	&+ \sum_{\boldsymbol{\alpha}=\boldsymbol{\beta}} \frac{1}{p^2}\delta_{\boldsymbol{\alpha}}\Vert\boldsymbol{\nu}_{\boldsymbol{\alpha}}\Vert_F^2\langle\cdot,\boldsymbol{\omega}_{\boldsymbol{\alpha}}\rangle\boldsymbol{\omega}_{\boldsymbol{\alpha}} = \frac{1}{p^2}\tilde{\mathcal{Q}}_{\Omega}^*\mathbf{H}^{-1}_{\boldsymbol{\omega}}\tilde{\mathcal{Q}}_{\Omega}(\cdot),
\end{align}
where $[\mathbf{H}^{-1}_{\boldsymbol{\omega}}]_{\boldsymbol{\alpha}\boldsymbol{\beta}}=\langle\boldsymbol{\nu}_{\boldsymbol{\alpha}},\boldsymbol{\nu}_{\boldsymbol{\beta}}\rangle$ denotes the $(\boldsymbol{\alpha},\boldsymbol{\beta})$ element in the correlation matrix of $\{\boldsymbol{\nu}_{\boldsymbol{\alpha}}\}_{\boldsymbol{\alpha}\in\mathbb{I}}$, and $[\mathbf{H}^{-1}_{\boldsymbol{\omega}}]_{\boldsymbol{\alpha}\boldsymbol{\alpha}}=\frac{n^2-2n+2}{2n^2}$ for any $\boldsymbol{\alpha}$\cite[Lemma A.6]{SmithCaiTasissaRieEDMC2}. \cite{SmithRieOptEDMCSubGChaos} then introduce the idea of diagonal de-biasing into \eqref{eq_RomegaConjRomega}, and construct the following sensing operator
\begin{align*}
	\setlength\belowdisplayskip{2pt}
	\setlength\abovedisplayskip{2pt}
	\frac{1}{p^2}&\mathcal{M}_{\Omega}(\cdot) = \underbrace{\frac{1}{p^2}\tilde{\mathcal{Q}}_{\Omega}^*\mathrm{Od}(\mathbf{H}^{-1}_{\boldsymbol{\omega}})\tilde{\mathcal{Q}}_{\Omega}(\cdot)}_{\mathcal{M}^1_{\Omega}(\cdot)} + \underbrace{\frac{1}{p}\Vert\boldsymbol{\nu}_{\boldsymbol{\alpha}}\Vert_F^2\tilde{\mathcal{Q}}_{\Omega}^*\tilde{\mathcal{Q}}_{\Omega}(\cdot)}_{M_2}.
\end{align*}
Two-sided concentration property of $M_2$ was established in\cite[Lemma B.1]{YCLiSunSNLRR} and further refined by\cite[Lemma B.6]{SmithRieOptEDMCSubGChaos}. While deriving analogies to Lemma \ref{lema_Romega_RIPL} and \ref{eq_NromalSpace_tight_bound} for $\mathcal{M}^1_{\Omega}(\cdot)$, either in terms of tangent space Restricted Isometry Property (RIP) or non-tangent space uniform bounds, is challenging and unsatisfactory due to at least two reasons. 

First, the desired estimates should hold uniformly over the set of low-rank matrices with nice incoherence\cite[Thm. 22]{JunPKrahmerBDB_OptimalRate}, in other words
\begin{align*}
	\setlength\belowdisplayskip{2pt}
	\setlength\abovedisplayskip{2pt}
	Z_{\mathcal{H}}(\boldsymbol{\xi}) :&=\sup_{\mathbf{X}\in\mathbb{T},\,\Vert\mathbf{X}\Vert_F=1}\bigg|\langle\mathbf{X},(\frac{1}{p^2}\mathcal{M}^1_{\Omega}-\mathbb{E}\frac{1}{p^2}\mathcal{M}^1_{\Omega})(\mathbf{X})\rangle\bigg| \\
	&=  \sup_{\mathbf{H}\in\mathcal{H}} |\boldsymbol{\delta}^T\mathbf{H}\boldsymbol{\delta}^T-\mathbb{E}(\boldsymbol{\delta}^T\mathbf{H}\boldsymbol{\delta}^T)|,
\end{align*} 
will be considered, where 
\begin{gather*}
	\setlength\belowdisplayskip{3pt}
	\setlength\abovedisplayskip{3pt}
	\boldsymbol{\delta} := [\delta_{(1,2)}/p,\dots,\delta_{(n-1,n)}/p]\in\mathbb{R}^{L},\\
	\mathbf{H}:= \mathbf{W}_\mathbf{X}\circ\mathrm{Od}(\mathbf{H}^{-1}_{\boldsymbol{\omega}}),\,[\mathbf{W}_\mathbf{X}]_{\boldsymbol{\alpha}\boldsymbol{\beta}} :=\langle\mathbf{X},\boldsymbol{\omega}_{\boldsymbol{\alpha}}\rangle\langle\mathbf{X},\boldsymbol{\omega}_{\boldsymbol{\beta}}\rangle, \\	
	\mathcal{H}:=\{\mathbf{W}_\mathbf{X}\circ\mathrm{Od}(\mathbf{H}^{-1}_{\boldsymbol{\omega}}):\mathbf{X}\in\mathbb{T},\Vert\mathbf{X}\Vert_F\leq 1\}.
\end{gather*}
Naively combining the $\epsilon$-net covering argument on set $\mathcal{H}$ with Hanson-Wright inequality\cite[Ch. 6]{VershyninHighDimPob} causes the tail to blow up. In literature, similar structures are often resolved by the ``suprema of chaos" inequality\cite{KrahmerMendSupremaChaos} for sub-Gaussian chaos generated by certain set of PSD matrices, which requires bounding Talagrand’s $\gamma_2$-functional\cite[Ch. 8.5]{VershyninHighDimPob}, e.g., \cite[Sec. VI-B]{JunPKrahmerBDB_OptimalRate}. While we conjecture that an analogous tangent space eigenvalue lower bound as in\cite[Lemma 10]{TasissaEDMCProof} for $\frac{1}{p^2}\mathcal{R}_{\Omega}^*\mathcal{R}_{\Omega}$ can be obtained using this technique, the de-biasing step causes $\mathbf{H}$ matrix no longer PSD, precluding a direct application of Krahmer's result\cite[Thm. 1.4]{KrahmerMendSupremaChaos}. Second, it is known that ``suprema of chaos" based non-tangent space upper bound over residual with bounded $l_{2,\infty}$ or $l_{\infty}$ norm\cite[Lemma 7, 8]{AhmedBDModInput} does not guarantee finite-sample exact recovery for reasons other than sample splitting. Unlike Lemma \ref{thm_main_local_contraction}, both the perturbation error and sample complexity are related to the Frobenius norm of the total residual during each step of contraction induction argument, see \cite[Sec. I-E]{AhmedBDModInput} for further discussion.

Finding alternatives to bypass the aforementioned roadblocks is beyond the scope of this work. But it is interesting to point out that $\frac{1}{p^2}\mathcal{M}_{\Omega}$ appears to exhibit better numerical stability than $\frac{1}{p}\mathcal{R}_{\Omega}$, as empirically reported in \cite[Sec. 8]{SmithRieOptEDMCSubGChaos}. We frame this as a choice between \textit{analytical tractability and algorithmic fidelity}, since our choice provides only marginal practical significance, whereas explaining the algorithm proposed in \cite{SmithRieOptEDMCSubGChaos} requires substantially more effort.
\section{Numerical Experiments}
\label{sec_numericalRes_Disscuss}
The algorithm tested here is slightly different from the theoretical routine listed in Theorem \ref{thm_main_global_contraction} (cf. Algorithm \ref{alg_APGD}), in which BB stepsize without line search is employed, as inherited from\cite[Ch. 2]{sun2015matrix}. The performance of incoherent projection-free OS-MDS, OS-MDS initialized GD refinement on s-stress (OS-MDS-GD)\cite{YCLiSunEDMC_Spec_init}, and OS-MDS initialized IFHT-EDMC (OS-MDS-IFHT)\cite[Alg. 2]{SmithCaiTasissaRieEDMC2} are also reported in Fig. \ref{fig_phase_transition_OS_MDS_SVD_MDS} and Fig. \ref{fig_OS_MDS_APGD_Phase_Trans}(e), respectively. 
\subsection{Implementation Details}
The parameters in performing $\mathcal{P}_I$ can be estimated from the observations
\begin{equation*}
	\setlength\belowdisplayskip{3pt}
	\setlength\abovedisplayskip{3pt}
	\hat{\sigma}_1^{\star}=c_1\sigma_1(\hat{\mathbf{P}}\hat{\mathbf{P}}^T),\hat{\mu}=c_2 \frac{n}{r\hat{\sigma}_1^{\star}}\max_{(i,j)\in\Omega}\mathbf{D}_{ij}^{\star},
\end{equation*}
where Lemma \ref{lema_OSMDS_Spec_init} guarantees the approximation behavior of $\hat{\sigma}_1^{\star}$, and $c_1$, $c_2$ are the parameters to tune. Since the APGD is more analogous to an analytical process rather than an algorithm used for benchmarking, we assume throughout all experiments that $\sigma_1^{\star}$ and $\mu$ are perfectly known. Unless otherwise specified, the ground truth point set is generated by a standard Gaussian distribution with embedding dimension $r=2$, and results presented here are obtained under $\eta_{\max}=10$\footnote{The lower safeguard on BB stepsize is removed in most experiments, i.e., $\eta_k$ is allowed to be negative. Compared with pure GD, we found such sporadically occurring gradient ascent can numerically improve overall global convergence in all tested scenarios (especially the protein tests in Section \ref{subsec_protein_tests}). The upper safeguard is retained to ensure algorithmic stability when $c_{IP}$ is large, c.f., \eqref{eq_incohrent_projection_C_I}.}. Algorithm \ref{alg_APGD} is terminated once $\Vert\mathbf{p}g^k\Vert_F\leq 1\times 10^{-6}$ or the iteration number reaches $N=1000$. In contrast, both the upper and lower safeguards, $\eta_{\max} = \frac{2}{p\sigma_1^{\star}\mu r\kappa}$ and $\eta_{\min} = \frac{1}{2p\sigma_1^{\star}\mu r\kappa}$, are adopted in Section \ref{subsec_IterationTrack} for better illustration of linear convergence. Please note that the calculation of $\mathbf{JAJ}$ for $\mathbf{A}\in\mathbb{R}^{n\times n}$ can be conducted in $\mathcal{O}(n^2)$ flops, as it amounts to sequentially removing the mean value of $\mathbf{A}$'s columns and rows.
\begin{algorithm}[!t]
	\caption{APGD With BB Stepsize}
	\label{alg_APGD}
	\begin{algorithmic}[1]
		\REQUIRE Pseudo-gradient $\mathbf{p}_g^k=p\hat{\nabla}f(\mathbf{P}_k)$, sampled distances $\mathcal{P}_{\Omega}\mathbf{D}^{\star}$. 
		\STATE $\hat{\mathbf{P}}\hat{\mathbf{P}}^T=\mathcal{T}_r\left[ -\frac{1}{2p} \mathbf{J} (\mathcal{P}_{\Omega}\mathbf{D}^{\star}) \mathbf{J}\right],\,\mathbf{P}^0=\mathcal{P}_I(\hat{\mathbf{P}})$
		\FOR{$k=0,1,\dots,N$}
		\STATE $\mathbf{P}^{k+1} = \mathbf{P}^k-\eta_k \mathbf{p}_g^k,\,\mathbf{P}^{k+1}=\mathcal{P}_I(\mathbf{P}^{k+1})$.
		\STATE $\mathbf{S}_{k+1} = \mathbf{P}^{k+1}-\mathbf{P}^k$, $\mathbf{D}_{k+1}= \mathbf{p}_g^{k+1}-\mathbf{p}_g^k$.
		\IF{$\text{mod}(k,2)==0$}
		\STATE $\eta_k^{BB}=\frac{\Vert\mathbf{S}_{k+1}\Vert_F^2}{\langle \mathbf{S}_{k+1},\mathbf{D}_{k+1} \rangle}$
		\ELSE 
		\STATE $\eta_k^{BB}=\frac{\langle \mathbf{S}_{k+1},\mathbf{D}_{k+1} \rangle}{\Vert\mathbf{D}_{k+1}\Vert_F^2}$
		\ENDIF
		\STATE $\eta_k = \min(\eta_k^{BB},\eta_{\max})$
		\ENDFOR
		\RETURN $\mathbf{P}^k$
	\end{algorithmic}
\end{algorithm}

The spectral error (SE) and EDM recover rate (RE) defined below are used to evaluate the performance of OS-MDS, OS-MDS-GD, and the APGD.
\begin{equation*}
	\setlength\belowdisplayskip{3pt}
	\setlength\abovedisplayskip{3pt}
	\mathrm{SE}:=\frac{\Vert\hat{\mathbf{G}}-\mathbf{G}^{\star}\Vert}{\Vert\mathbf{G}^{\star}\Vert},\,\mathrm{RE}:=\frac{\Vert\bar{\mathbf{D}}-\mathbf{D}^{\star}\Vert_F}{\Vert\mathbf{D}^{\star}\Vert_F},
\end{equation*}
where $\hat{\mathbf{G}}$ is the surrogate Gram matrix obtained from the spectral initialization without trimming, and $\bar{\mathbf{D}}$ is the EDM returned by either OS-MDS-GD or Algorithm \ref{alg_APGD}. We record the trajectories of the following quantities during the execution of APGD.
\begin{subequations}
	\setlength\belowdisplayskip{3pt}
	\setlength\abovedisplayskip{3pt}
	\label{eq_record_track_ingrd}
	\begin{gather}
		\mathrm{g}_1^k = 2\Vert\mathcal{R}_{\Omega}(\mathbf{P}^k\mathbf{P}^{k T}-\mathbf{G}^{\star})\mathbf{P}^k\Vert_F,\\
		\mathrm{g}_2^k = 2\Vert\mathcal{R}_{\Omega}^*(\mathbf{P}^k\mathbf{P}^{k T}-\mathbf{G}^{\star})\mathbf{P}^k\Vert_F,\\
		\mathrm{r}^k = \frac{\Vert\mathbf{P}^k\mathbf{P}^{k T}-\mathbf{G}^{\star}\Vert_F}{\Vert\mathbf{G}^{\star}\Vert_F},
	\end{gather}
\end{subequations}
where $\mathrm{g}_2^k$ can be viewed as the norm of the second term in the gradient of \eqref{eq_fracted_pre_cond_cost_func}, it cannot be computed unless direct access to $\mathbf{G}^{\star}$ is available. We conjecture that the limited numerical performance of APGD in the sample-limited regime can be attributed to at least three factors: (i) the weakening of implicit regularization; (ii) the large residual between the pseudo-gradient and its population-level counterpart; (iii) the inconsistency between pseudo-gradient and the true gradient of \eqref{eq_fracted_pre_cond_cost_func}. These factors will be illustrated sequentially in the following sections.
\begin{figure}[t]
	\centering
	\setlength{\abovecaptionskip}{0cm}
	\setlength{\belowcaptionskip}{-3cm}
	\includegraphics[width=9cm]{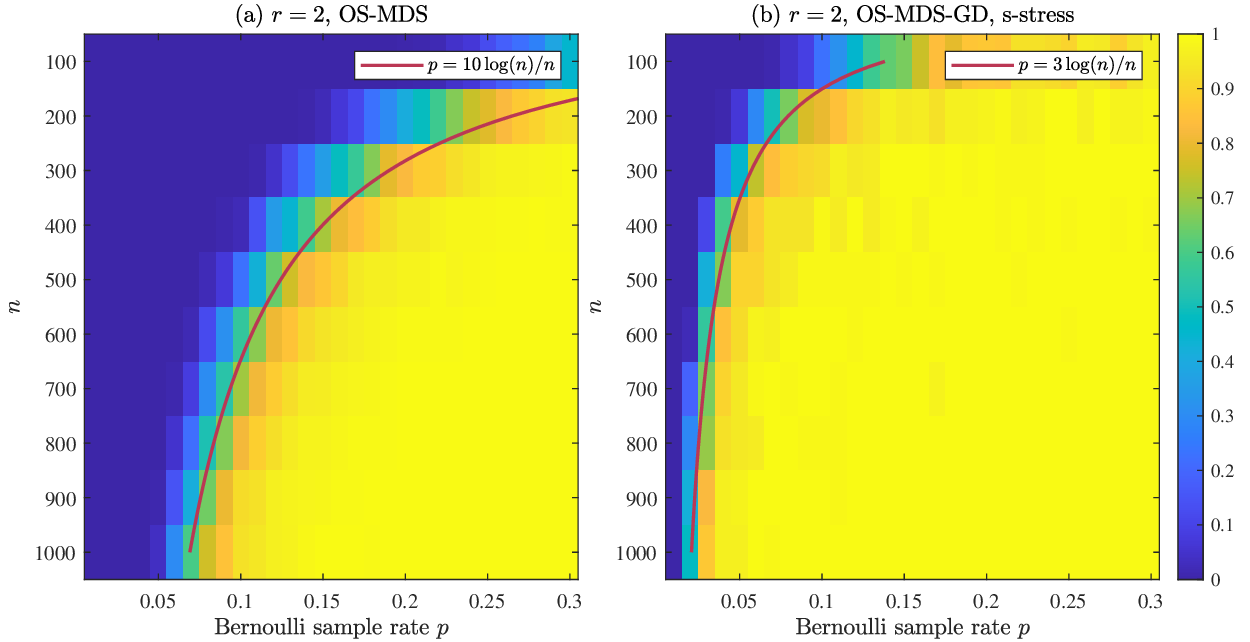}
	\caption{The phase transition of regularization free OS-MDS and OS-MDS-GD when varying $p$ and $n$\cite{YCLiSunEDMC_Spec_init}. We claim a success if the spectral error falls below $1$ in (a), or the EDM recover rate is smaller than $10^{-3}$ in (b). For each $(n,p)$, the result is obtained by 100 independent trials.}
	\label{fig_phase_transition_OS_MDS_SVD_MDS}
\end{figure}
\begin{figure*}[t]
	\centering
	\setlength{\abovecaptionskip}{0cm}
	\setlength{\belowcaptionskip}{-3cm}
	\includegraphics[width=18.5cm]{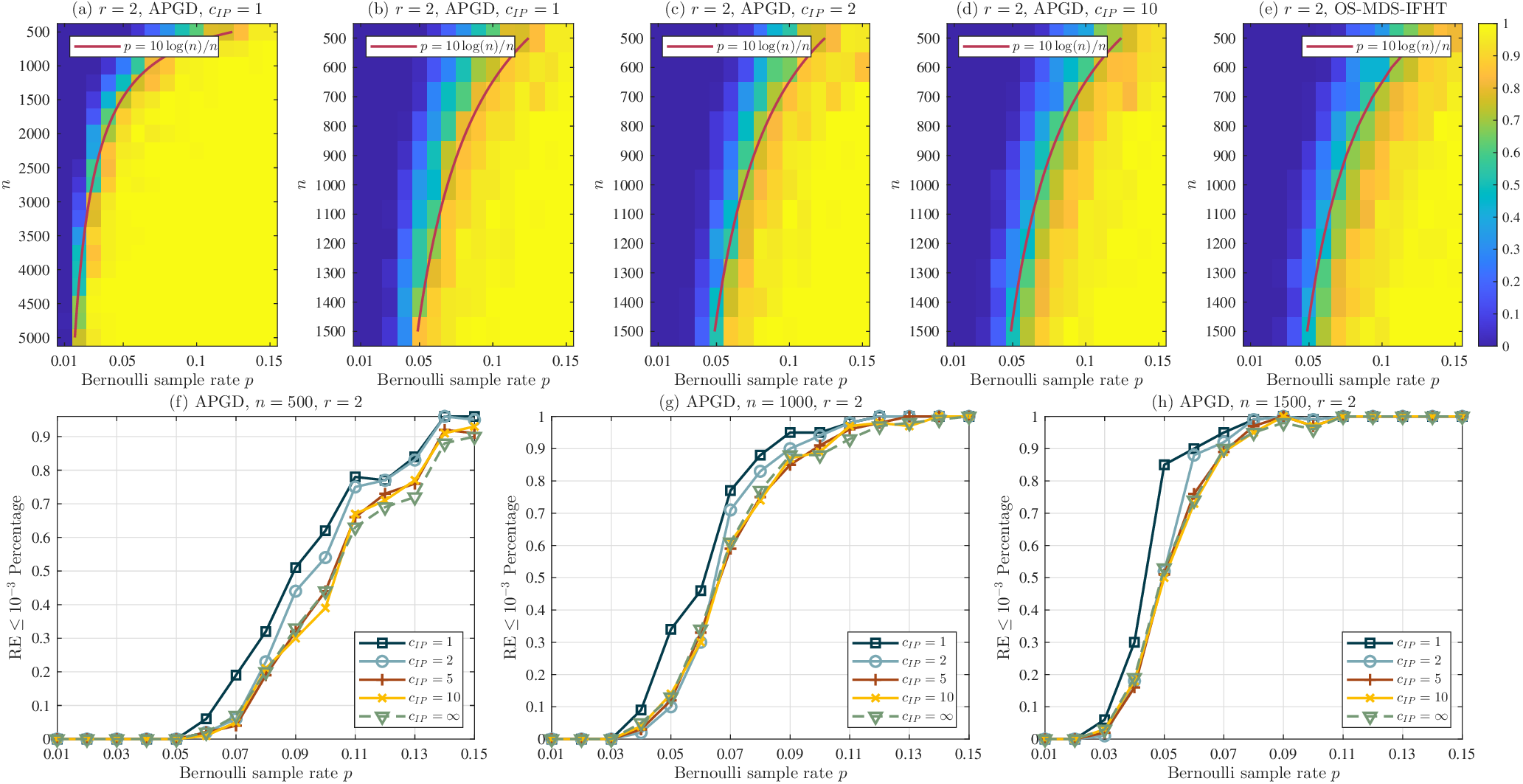}
	\caption{(a) The phase transition of APGD when varying Bernoulli sample rate $p$ and the number of points $n$. The result is obtained by 100 independent trials. To verify Theorem \ref{thm_main_global_contraction}, we claim a success if the EDM recover rate falls below $10^{-3}$. In (b), (c), and (d), we vary the constant $c_{IP}$ in \eqref{eq_incohrent_projection_C_I} when performing incoherent projection, while all other parameters remain unchanged. (e) depicts the phase transition of IFHT-EDMC algorithm \cite[Alg. 2]{SmithCaiTasissaRieEDMC2} under the same problem setup as in (b). Subplots (f), (g), and (h) show selected phase transition curves when $n=500, 1000, 1500$ with $c_{IP}$ varies.}
	\label{fig_OS_MDS_APGD_Phase_Trans}\vspace{-10pt}
\end{figure*}
\begin{figure*}[t]
	\raggedright
	\setlength{\abovecaptionskip}{0cm}
	\setlength{\belowcaptionskip}{-3cm}
	\includegraphics[width=18.5cm]{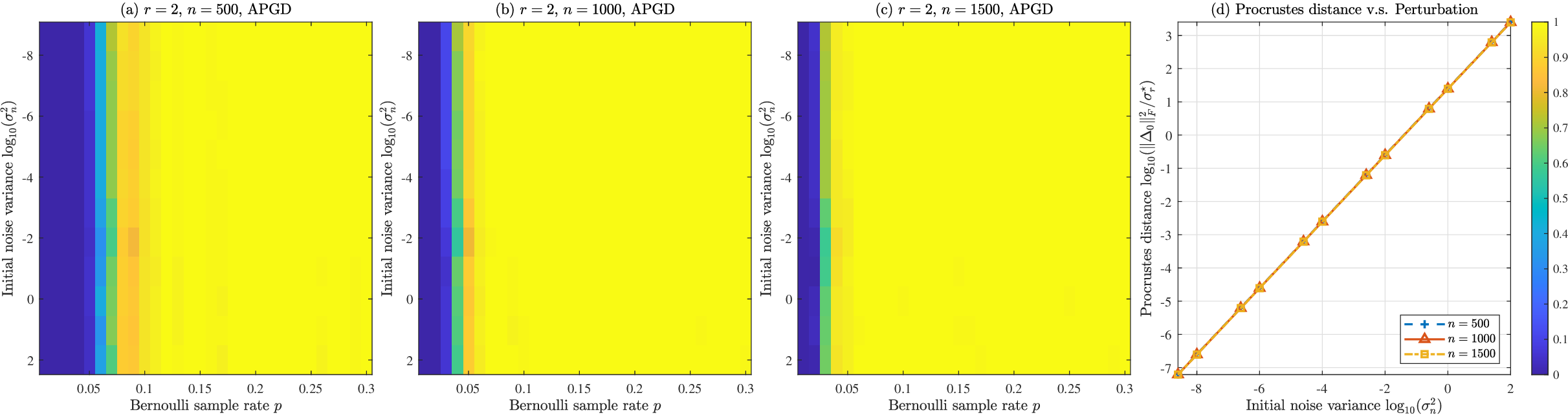}
	\caption{The phase transition of APGD when varying Bernoulli sample rate $p$ and the standard variance of perturbation white noise $\sigma_n$. A success is recorded if the EDM recover rate falls below $10^{-5}$. In (a), (b), and (c), we change the number of points from $n=500$ to $n=1500$, with all other parameters fixed. (d) plots the relationship between normalized quotient distance $\Vert\Delta\Vert_F^2/\sigma_r^{\star}$ and the intensity of noise $\sigma_n^2$.}
	\label{fig_OS_MDS_APGD_Random_Perturbate}\vspace{-10pt}
\end{figure*}
\begin{figure*}[t]
	\raggedright
	\setlength{\abovecaptionskip}{0cm}
	\setlength{\belowcaptionskip}{-3cm}
	\includegraphics[width=18.5cm]{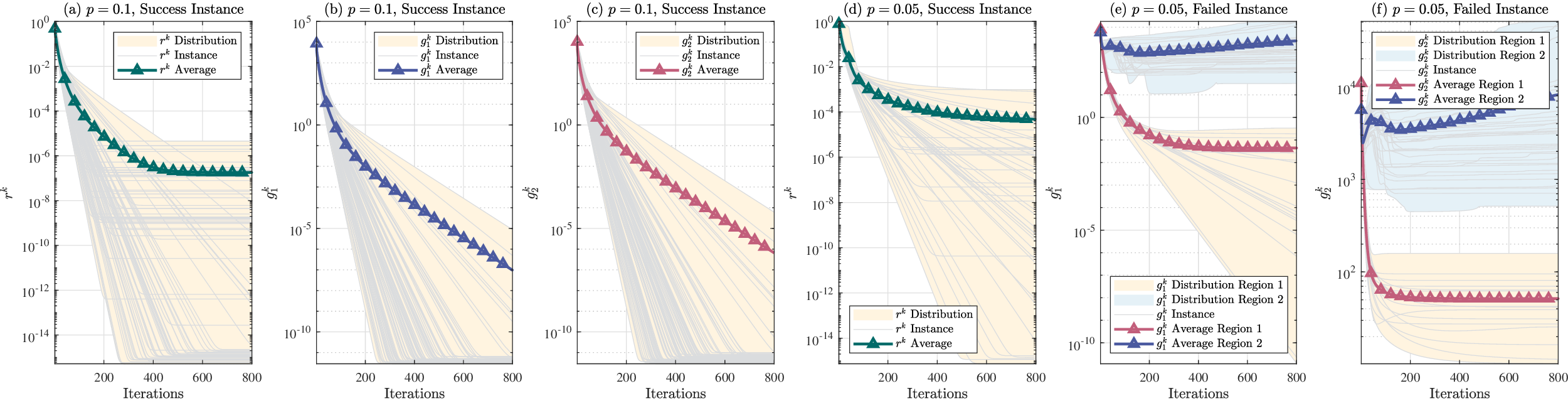}
	\caption{Trajectories of all three quantities in \eqref{eq_record_track_ingrd} while fixing $n=1500$. The result is obtained from 100 independent trials. (a), (d) shows the tendency of $r^k$ in selected succeed instances (96\% when $p=0.1$, 47\% when $p=0.05$). (b), (c) plots the corresponding $g_1^k$, $g_2^k$ when $p=0.1$. (e), (f) contains the records of $g_1^k$ and $g_2^k$ of selected failed instances (53\% when $p=0.05$). These trajectories can be divided into two groups (17\% in region 1, 36\% in region 2).}
	\label{fig_OS_MDS_APGD_iteration_Track}\vspace{-10pt}
\end{figure*}
\subsection{Phase Transition Under Gaussian Point Set}
\label{subsecPT_under_Gaussian_point}
The phase transition of APGD is plotted in Fig. \ref{fig_OS_MDS_APGD_Phase_Trans}, and we fit its edge using $p=10\log(n)/n$ curve. Interestingly, the incoherence projection becomes essential for achieving relatively robust performance, which bears a resemblance to the case in standard LRMC problems \cite{SL15NonCVXFR}. If one sets
\begin{equation}
	\setlength\belowdisplayskip{3pt}
	\setlength\abovedisplayskip{3pt}
	\label{eq_incohrent_projection_C_I}
	\Vert\mathbf{P}_{k}(i,:)\Vert_2\leq c_{IP}\sqrt{\frac{\mu r\sigma_1^{\star}}{n}},
\end{equation}
in \eqref{eq_the_trimming_step} with some $c_{IP}>1$, the performance of APGD on Gaussian random point sets will be degenerated, as suggested by Fig. \ref{fig_OS_MDS_APGD_Phase_Trans}(b), (c), and (d). This observation indicates that the power of implicit regularization to maintain incoherence is weakened in the APGD case. More delicate evaluations are carried out in Fig. \ref{fig_OS_MDS_APGD_Phase_Trans}(e), (f), and (g). This reveals that APGD can still converge when the trimming step is removed, yet larger $c_{IP}$ gives rise to more oscillations and non-monotonicity in the transition curve, ultimately leading to broader transition interval measured by $p$. Such conclusion is also supported by Fig. \ref{fig_OS_MDS_APGD_Phase_Trans}(e), where the OS-MDS-IFHT\cite{SmithCaiTasissaRieEDMC2}, i.e., regularization-free embedding geometry counterpart of APGD, exhibits a resemblance in its transition edge to that in Fig. \ref{fig_OS_MDS_APGD_Phase_Trans}(d) \footnote{Since IFHT is a tricky variant of Singular Value Projection algorithm\cite{JainSVP}, which eliminates the necessity of proving the iteration can automatically stay incoherent\cite[Sec. III-B]{DingChenLOOPrimalDual}.}.

By comparing Fig. \ref{fig_phase_transition_OS_MDS_SVD_MDS}(a) and Fig. \ref{fig_OS_MDS_APGD_Phase_Trans}(a), we find that the phase transition edge of APGD matches that of OS-MDS. This phenomenon confirms the prediction in Lemma \ref{thm_main_local_contraction}, and yet is quite unusual, as both non-convex Gaussian ensemble phase retrieval and structure-less LRMC do not necessarily rely on the spectral initialization to succeed\cite{ChenYXPRRandominit}\cite{pmlr-v247-ma24a}. Since the proof roadmap in\cite{pmlr-v247-ma24a} indicates that the success of rescaled eigenspace alignment initialization relies on the implicit incoherent regularization ability of the original fourth-order polynomial-type cost function for LRMC, we infer that when this property is compromised, the algorithm may start to behave selectively on the quality of starting point. Alternatively, more dedicated initialization for APGD can somewhat improve the transition edge. However, the OS-MDS-GD strategy exhibits empirical convergence when started inside a slack region around the ground truth as depicted in Fig. \ref{fig_phase_transition_OS_MDS_SVD_MDS}(b), even without any incoherence projection.
\begin{table*}[t]
	\centering
	\renewcommand\arraystretch{1.2}
	\caption{Parameters of Test Proteins}
	\label{table_para_of_protein}
	\scalebox{0.95}{\begin{tabular}{ccccc|ccccc|ccccc}
			\hline
			\textbf{Protein} & $n$    & $\mu$    & $\kappa$ & $\mu\kappa$ & \textbf{Protein} & $n$    & $\mu$    & $\kappa$ & $\mu\kappa$ & \textbf{Protein} & $n$     & $\mu$    & $\kappa$  & $\mu\kappa$ \\ \hline
			\textbf{1AX8}         & $1003$ & $1.8653$ & $2.9567$ & $5.5154$    & \textbf{1KDH}         & $2846$ & $2.2583$ & $2.7716$ & $6.2592$    & \textbf{1BPM}         & $3671$  & $1.9289$ & $4.1451$  & $7.9956$    \\
			\textbf{1RGS}         & $2015$ & $2.0494$ & $4.3926$ & $8.8733$    & \textbf{1MQQ}         & $5681$ & $2.1796$ & $3.3211$ & $7.2389$    & \textbf{1YGP}         & $13488$ & $1.6348$ & $3.8056$  & $6.2218$    \\
			\textbf{1TIM}         & $3740$ & $1.3297$ & $4.9413$ & $6.5706$    & \textbf{1TOA}         & $4292$ & $1.6424$ & $5.3534$ & $8.7929$    & \textbf{1I7W}         & $8629$  & $1.8189$ & $20.5725$ & $37.4202$   \\ \hline
	\end{tabular}}
\vspace{-12pt}
\end{table*}
\subsection{Phase Transition Under Random Perturbation}
\label{subsec_PT_Reg_piont_Gau_Pertur}
A randomly perturbed initialization is deployed in this section to verify the existence of restricted strong convexity in the pseudo-gradient direction. It reads
\begin{equation*}
	\setlength\belowdisplayskip{3pt}
	\setlength\abovedisplayskip{3pt}
	\mathbf{P}_0 = \mathbf{J}(\mathbf{P}^{\star}+\sigma_n \mathbf{E}_n)\in\mathbb{R}^{n\times r},
\end{equation*}
where $\mathbf{E}_n$ is a Gaussian random matrix with $[\mathbf{E}_n]_{ij}\sim\mathcal{N}(0,1)$ and $\mathbf{P}^{\star}$ is generated from a uniform distribution over the hypercube $[-0.5,0.5]^r$. The noise level $\sigma_n$ is varied from $5\times 10^{-5}$ to $10$. The iteration will be terminated if $\Vert\mathbf{p}_g^k\Vert_F\leq 10^{-8}$, while all other parameters follow the same setup as in Fig. \ref{fig_OS_MDS_APGD_Phase_Trans}(b). Since the perturbation can be relatively small, we claim a success if $\mathrm{RE}\leq 10^{-5}$, and the resulting phase transitions are plotted in Fig. \ref{fig_OS_MDS_APGD_Random_Perturbate} for $n=500, 1000, 1500$, $r=2$, where $\Vert\Delta_0\Vert_F^2$ refers to $\mathrm{dist}(\mathbf{P}_0,\mathbf{P}^{\star})^2$. Around $\sigma_n\leq 0.1$, i.e., inside the region where $\Vert\Delta_0\Vert_F^2\leq \sigma_r^{\star}$ according to Fig. \ref{fig_OS_MDS_APGD_Random_Perturbate}(d), these plots depict that the convergence of APGD becomes sensitive primarily to the sample complexity. As $p\to 1$, the $\hat{\nabla} f(\mathbf{P})$ reduces to the gradient direction of the vanilla matrix factorization problem, which is known to exhibit local restricted strong convexity\cite[Thm. 5]{ZhuGlobalGeo}. Therefore, Fig. \ref{fig_OS_MDS_APGD_Random_Perturbate} also demonstrates the lower bound on $p$ necessary for this approximation to hold. 

However, this dependence on $p$ is sub-optimal compared to the performance achieved by either OS-MDS-GD or trace minimization methods\cite[Sec. IV]{TasissaEDMCProof}\cite[Sec. 5]{GhoshTasissaIRLSEDMC}, indicating that the stabilization of pseudo gradient, e.g., ensuring the upper bound on Eq. \eqref{eq_Sample_to_Exp_Distor_4term} in Appendix \ref{Appdi_B1_Bound}, needs more random samples to guarantee, even though Lemma \ref{thm_main_local_contraction} predicts that the establishment of regularity condition should require the same order of sample complexity as in classical LRMC problems. Moreover, it is noteworthy that OS-MDS cannot achieve $\Vert\Delta_0\Vert_F^2\leq \sigma_r^{\star}$ with high probability at the sample configurations $\{(n,p):(500,0.08),(1000,0.05),(1500,0.04)\}$, which is why we allow Algorithm \ref{alg_APGD} to take negative stepsize, thereby facilitating global optimization. 
\subsection{Iteration Trajectory}
\label{subsec_IterationTrack}
A special instance of experiments in Section \ref{subsecPT_under_Gaussian_point} is investigated in greater detail. We consider $(n,p)=(1500,0.1)$, $(n,p)=(1500,0.05)$, and plot the convergence trajectories measured by \eqref{eq_record_track_ingrd} in Fig. \ref{fig_OS_MDS_APGD_iteration_Track}. When $p=0.1$, 96\% of the trials successfully recover the EDM. The corresponding trajectories of $r^k$, $g^k_1$, and $g^k_2$ are presented in Fig. \ref{fig_OS_MDS_APGD_iteration_Track}(a), (b), and (c) respectively, where exact linear convergence is observed after a few iterations, as indicated by the gray lines. Moreover, the pseudo gradient and $g^k_2$ demonstrated nearly identical behavior. When $p=0.05$, 47\% of the recordings are categorized to be successful, with their $r^k$ shown in Fig. \ref{fig_OS_MDS_APGD_iteration_Track}(d). Compared with Fig. \ref{fig_OS_MDS_APGD_Phase_Trans}(h), this clearly highlights the performance degradation caused by imposing the lower safeguard $\eta_{\min}$. Also, the succeeded instances of $g^k_1$ and $g^k_2$ when $p=0.05$ resemble those in Fig. \ref{fig_OS_MDS_APGD_iteration_Track}(b), (c), and thus omitted here. Intriguing phenomena emerge when analyzing the other failed $53\%$, as illustrated in Fig. \ref{fig_OS_MDS_APGD_iteration_Track}(e), (f), where $g^k_1$ and $g^k_2$ no longer coincide. In about $17\%$ of all trials, the pseudo gradient tends to converge, whereas the corresponding $g^k_2$ does not. This observation might uncover a potential factor underlying the degraded numerical performance of APGD, i.e., the consistency between pseudo gradient and the real gradient of \eqref{eq_fracted_pre_cond_cost_func} (up to trivial rescale) demands an unreasonably large number of samples to be reliably ensured.
\subsection{Protein Test}
\label{subsec_protein_tests}
Phase transition tests on nine proteins downloaded from the Protein Data Bank\cite{BermanProteinDataBank} are presented in Fig. \ref{fig_APGD_Protein_test_9All_no_noise} to evaluate the real-world performance of APGD, as well as the impact of condition number $\kappa$ and coherence parameter $\mu$. Similar experiments but under the unit ball sample model and challenging noisy observations settings can be found in\cite[Sec. 5.3]{LeungTohSDPDCMP}. Prior to processing, solvent molecules and chelated metal ions were removed from the datasets. The resulting parameters $n$, $\mu$, and $\kappa$ are reported in Table \ref{table_para_of_protein}. 

Non-monotonic trends in sample complexity $p$ w.r.t. $n$ are observed, i.e., the phase transition for a protein with larger $n$ occurs later than that for a protein with smaller $n$. For example, although the atom number of 1I7W is approximately three times larger than that of 1KDH, the phase transition for 1KDH occurs at a significantly lower sampling ratio. The large condition number associated with the 1I7W configuration suppresses the BB stepsize, and increases the sample complexity required for global recovery, as predicted by Theorem \ref{thm_main_global_contraction}. Furthermore, the performance of Algorithm \ref{alg_APGD} appears sensitive to the product $\mu\kappa$, since the curves of 1AX8 versus 1RGS, and 1TIM versus 1TOA exhibit similar ``reversal", due to small changes in their corresponding $\mu\kappa$ values.
\section{Conclusion and Discussion}
\label{sec_conclude_Diss}
By paralleling matrix factorization type incoherent LRMC technique, this manuscript analyzes the quotient variant of a recent proposed Iterative Fast Hard Thresholding procedure for Euclidean Distance Matrix Completion\cite{SmithCaiTasissaRieEDMC2}, namely, the Asymmetric Projected Gradient Descent. By leveraging the nuclear norm splitting trick in \cite{pmlr-v32-bhojanapalli14}\cite{pmlr-v48-lii16}\cite{ChenJiLiModelFreeMC}\cite{DingChenLOOPrimalDual}, and employing the sharp bound on the second largest eigenvalue of a Erd\"{o}s-R\'{e}nyi graph\cite{BandeiraHandelSharpRGL}, we provide refined analysis on different components of the residual, and further establish the regularity condition. This, in turn, enables the first characterization of near-linear convergence behavior for matrix factorization-based EDMC algorithms. Numerical experiments corroborate the predicted convergence behavior of the pseudo gradient descent method in rich-sample regimes, yet raise questions on general design principle and the adequacy of classic incoherence notion in the context of non-convex EDMC techniques, we list two of them below.
\begin{itemize}
	\item Since the EDM basis shares a similar support pattern with $\mathbf{e}_i\mathbf{e}_j^T$, \cite{TasissaEDMCProof} introduced the incoherence w.r.t. $\boldsymbol{\omega}_{\boldsymbol{\alpha}}$ and $\boldsymbol{\nu}_{\boldsymbol{\alpha}}$ to measure the concentration of information in $\mathbf{G}^{\star}$. While such definition, c.f.,\cite[Def. 2]{TasissaLRMC_genebasis}, can be deduced from classical incoherence assumption up to trivial rescale by resorting to\cite[Lemma A.4]{SmithCaiTasissaRieEDMC2}, this does not fully account for the observed performance gap between trace minimization and VGD applied to \eqref{eq_s_stress_forthorder_form}. In the context of Hankel Matrix Completion (HMC), incoherence is satisfied if the minimum wraparound distance between spikes remains non-vanishing\cite[Sec. III-A]{ChenChiSpectCompreHMC}, and both EMaC (convex approach to HMC) and PGD-like algorithms\cite{LiCuiPGD_SpectralCSHankel}\cite{MaoChenBlindSR_PGD} are observed to succeed even without strict separation. But for EDMC, which property of the point set govern the problem's tractability? 
	\item According to\cite{TasissaLRMC_genebasis}, the dual basis approach provides a way to construct a preconditioned sampling operator for non-orthonormal, discrete, and ill-conditioned sensing bases. Corresponding modifications of the Golfing Scheme, IFHT, and PDG frameworks achieve theoretical success for EDMC, yet the practical significance of the latter two severely suffers. While it is acknowledged that under the orthogonal setting, numerical performance of similar strategies relies on implicit regularization, which characteristic of the sample basis determines its strength?
\end{itemize}
\begin{figure}[t]
	\centering
	\setlength{\abovecaptionskip}{0cm}
	\setlength{\belowcaptionskip}{-3cm}
	\includegraphics[width=9cm]{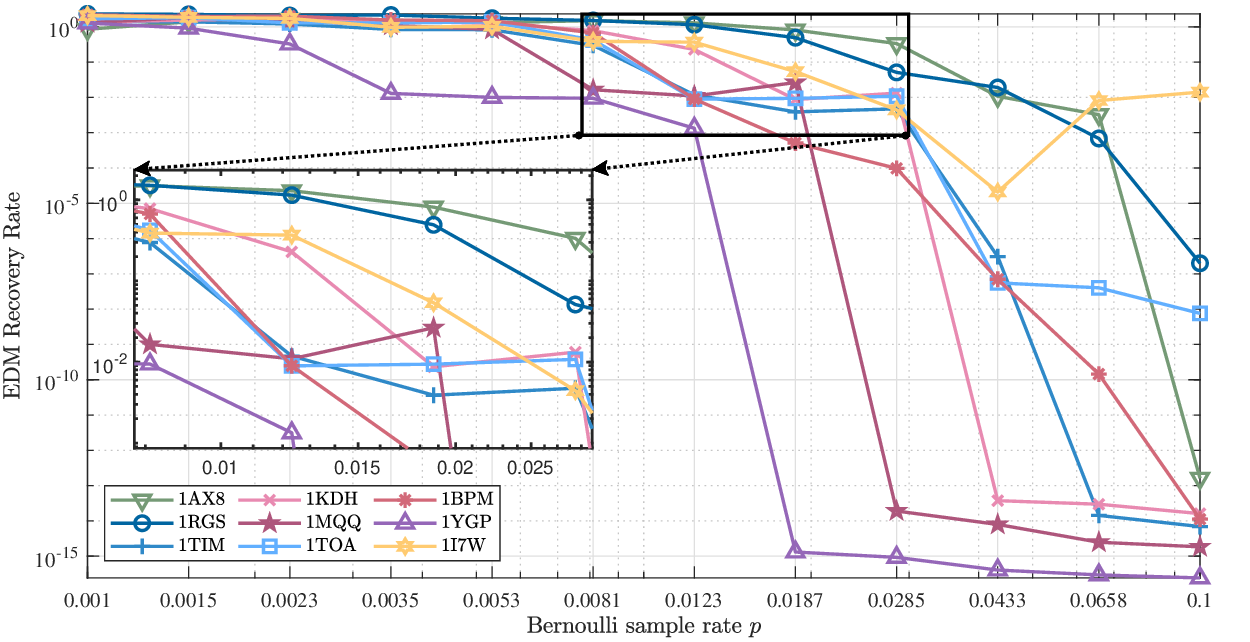}
	\caption{The phase transition measured by EDM recovery rate of APGD when varying protein type and Bernoulli sample rate $p$, while fixing $r=3$. Detailed parameters for each protein are listed in Table \ref{table_para_of_protein}. The results are averaged over $5$ independent trials.}
	\label{fig_APGD_Protein_test_9All_no_noise}
\end{figure}
\section*{ACKNOWLEDGMENT}
\addcontentsline{toc}{section}{ACKNOWLEDGMENT}
The authors would like to thank Abiy Tasissa for insightful discussion and help in proofreading the manuscript. They would also like to thank the anonymous reviewers and associate editor for their constructive comments. Y. L. would like to thank Jianguo Liu for proofreading during the revision.
{\appendices
	\section{Proof of Lemma \ref{eq_OS_MDS_the_basic_case}}
	\label{eq_proof_of_OS_mds}
	This is a direct use of Lemma \ref{lema_OSMDS_Spec_init}. Notice that
	\begin{align*}
		\setlength\belowdisplayskip{3pt}
		\setlength\abovedisplayskip{3pt}
		\Vert\Delta_1\Vert_F^2 :&= 	\Vert\mathcal{P}_{I}(\hat{\mathbf{P}}-\mathbf{P}^{\star}\boldmath{\psi}^{\star})\Vert_F^2\overset{(i)}{\leq} \Vert\hat{\mathbf{P}}-\mathbf{P}^{\star}\boldmath{\psi}^{\star}\Vert_F^2\\
		&\overset{(ii)}{\leq} 	\frac{1}{2(\sqrt{2}-1)\sigma_r^{\star}}\Vert\hat{\mathbf{P}}\hat{\mathbf{P}}^T-\mathbf{P}^{\star}\mathbf{P}^{\star T}\Vert_F^2\\
		&\overset{(iii)}{\leq}\frac{r}{2(\sqrt{2}-1)\sigma_r^{\star}}\Vert\hat{\mathbf{P}}\hat{\mathbf{P}}^T-\mathbf{P}^{\star}\mathbf{P}^{\star T}\Vert^2\\
		&\overset{(iv)}{\leq}\frac{r\varepsilon^2\sigma_1^{\star 	2}}{2(\sqrt{2}-1)\sigma_r^{\star}}\overset{(v)}{\lesssim}\frac{\sigma_r^{\star}}{c_c},
	\end{align*}
	where $(i)$ uses the convexity of the set $\{\mathbf{P}:\Vert\mathbf{P}\Vert_{2,\infty}\leq \sqrt{\frac{c_I\mu r\sigma_1^{\star}}{n}}\}$, see also \cite[Lemma 11]{ZhengLaffertyNonCVXFR}. $(ii)$ comes from \cite[Lemma 6]{RongGeSpuriousLocalMinima}, and $(iii)$ uses $\Vert\mathbf{A}\Vert_F^2\leq r\Vert\mathbf{A}\Vert^2$ if $\mathrm{rank}(\mathbf{A})=r$. $(iv)$ comes from Lemma \ref{lema_OSMDS_Spec_init}, and holds with probability at least $1-n^{1-\beta}$ given $p\gtrsim \frac{C_{\beta}\mu^2 r^2 \log n}{\varepsilon^2 n}$. Finally $(v)$ by setting $\varepsilon^2\leq \frac{1}{C\kappa^2 r}$, where $C\gtrsim c_c$, concluding the proof.\qed 
	\section{Proof of Lemma \ref{thm_main_local_contraction}}
	\label{Appdi_B_Bound}
	Recall the rank-$r$ SVD of $\mathbf{G}^{\star}=\mathbf{P}^{\star}\mathbf{P}^{\star T}$ is denoted as $\mathbf{U}^{\star}\boldsymbol{\Sigma}^{\star}\mathbf{U}^{\star T}$. The tangent space at $\mathbf{G}^{\star}$ are given by \eqref{eq_embedded_space}, and the projections onto it is denoted by $\mathcal{P}_{\mathbb{T}}$.
	\begin{subequations}
		\setlength\belowdisplayskip{4pt}
		\setlength\abovedisplayskip{4pt}
		\label{eq_embedded_space}
		\begin{gather}
			\label{eq_tangent_space}
			\mathbb{T}=T_{\mathbf{G}^{\star}}\mathcal{S}^{r,n}_+=\{\mathbf{U}^{\star}\mathbf{W}_1^T+\mathbf{W}_1\mathbf{U}^{\star T}\},\\
			\label{eq_norm_space}
			\mathcal{P}_{\mathbb{T}}(\mathbf{Y})=\mathcal{P}_{\mathbf{U}}\mathbf{Y}+\mathbf{Y}\mathcal{P}_{\mathbf{U}}-\mathcal{P}_{\mathbf{U}}\mathbf{Y}\mathcal{P}_{\mathbf{U}},
		\end{gather}
	\end{subequations}
	where $\mathcal{P}_{\mathbf{U}}=\mathbf{U}^{\star}\mathbf{U}^{\star T}$, and $\mathbf{W}_1\in\mathbb{R}^{n\times r}$ are arbitrary. Let $\hat{\nabla}f(\mathbf{P}):=\frac{2}{p}\mathcal{R}_{\Omega}(\mathbf{P}\mathbf{P}^T-\mathbf{G}^{\star})\mathbf{P}$ denote the pseudo gradient. We will also frequently use Lemma \ref{lema_AB_spec_F_norm_to_SingVal} without explicitly referring to it.
	\subsection{Local Contraction Without Re-sampling}
	\label{AppdiAA_local_contraction}
	It is meanwhile straightforward to prove the local contraction property by following the roadmap in, e.g., \cite[App. D]{ZhengLaffertyNonCVXFR}\cite[App. C]{LiCuiPGD_SpectralCSHankel}\cite{MaoChenBlindSR_PGD}, given the following claim holds for the moment. 
	\newtheorem{claim}{Claim}[section]
	\begin{claim}
		\label{claim_regularity_cond}
		For a fix ground truth point set $\mathbf{P}^{\star}$ independent with $\Omega$, uniformly for all $\mathbf{P}\in\mathbb{R}^{n\times r}$ inside the RIC $\mathcal{B}$, we have the following regularity condition holds
			\begin{subequations}
				\setlength\belowdisplayskip{4pt}
				\setlength\abovedisplayskip{4pt}
				\label{eq_regualirty_conds}
				\begin{gather}
					\label{eq_re_strongCVX}
					\langle \hat{\nabla}f (\mathbf{P}),\Delta \rangle\geq \alpha \Vert\Delta\Vert_F^2,\\
					\label{eq_re_RE_Smooth}
					\Vert\hat{\nabla}f (\mathbf{P})\Vert_F^2 \leq \rho\Vert\Delta\Vert_F^2,\\
					\label{eq_regularity_cond}
					\langle \hat{\nabla}f (\mathbf{P}),\Delta \rangle \geq \frac{\alpha}{2}\Vert\Delta\Vert_F^2+\frac{\alpha}{2\rho}\Vert\hat{\nabla}f (\mathbf{P})\Vert_F^2,
				\end{gather}
			\end{subequations}		
	with probability at least $1-cn^{1-\beta}$, provided with $p\gtrsim C_{\beta}\frac{\kappa^2 \mu^2 r^2\log n}{\varepsilon^2 n}$. 
	Where $\alpha=\frac{1}{2}\sigma_r^{\star}$, $\rho=84\sigma_1^{\star 2}c_I\mu r$, and $\varepsilon\leq\frac{1}{10}$. 
	\end{claim}
	\eqref{eq_re_strongCVX} and \eqref{eq_re_RE_Smooth} will be proved in App. \ref{Appdi_B1_Bound} and \ref{Appdi_B2_Bound}, respectively. \eqref{eq_regularity_cond} is a direct conclusion from combining \eqref{eq_re_strongCVX} and \eqref{eq_re_RE_Smooth}. Therefore we have
	\begin{align}
		\setlength\belowdisplayskip{3pt}
		\setlength\abovedisplayskip{3pt}
	\Vert\Delta_{k+1}\Vert_F^2&\leq\Vert\mathbf{P}_{k+1}-\mathbf{P}^{\star}\boldsymbol{\psi}^{\star}_k\Vert_F^2\nonumber\\
		&=\Vert\mathcal{P}_{I}(\mathbf{P}_k-\eta\hat{\nabla}f(\mathbf{P}_k)-\mathbf{P}^{\star}\boldsymbol{\psi}^{\star}_k)\Vert_F^2\nonumber\\
		&\leq\Vert\mathbf{P}_k-\eta\hat{\nabla}f(\mathbf{P}_k)-\mathbf{P}^{\star}\boldsymbol{\psi}^{\star}_k\Vert_F^2\nonumber\\
		&=\Vert\Delta_k\Vert_F^2+\eta^2\Vert\hat{\nabla}f(\mathbf{P}_k)\Vert_F^2-2\eta\langle\hat{\nabla}f(\mathbf{P}_k),\Delta_k\rangle\nonumber\\
		&\overset{(i)}{\leq} (1-\eta\alpha)\Vert\Delta_k\Vert_F^2\nonumber+\eta(\eta-\frac{\alpha}{\rho})\Vert\hat{\nabla}f(\mathbf{P}_k)\Vert_F^2\nonumber\\
		\label{eq_local_contraction}
		&\overset{(ii)}{\leq}(1-\eta\alpha)\Vert\Delta_k\Vert_F^2,
	\end{align}
	$(i)$ by using Claim \ref{claim_regularity_cond}, and $(ii)$ by setting $0<\eta\leq\min\{\frac{1}{\alpha},\frac{\alpha}{\rho}\}=\frac{1}{168\sigma_1^{\star}c_I\mu r\kappa}$, where we conclude the proof.
	\qed
	\subsection{Restricted Strong Convexity: \eqref{eq_re_strongCVX}}
	\label{Appdi_B1_Bound}
	The proof follows similar idea as in \cite[App. B]{YCLiSunSNLRR}, but the asymmetry structure of the pseudo gradient requires very careful control of the residuals. We start from lower bounding the population level gradient. Let us denote $\mathbb{E}\hat{\nabla}f(\mathbf{P})=2(\mathbf{PP}^T-\mathbf{G}^{\star})\mathbf{P}$, and $\bar{\mathbf{P}}^{\star}=\mathbf{P}^{\star}\boldsymbol{\psi}^{\star}$, thus
	\begin{align}
		\setlength\belowdisplayskip{3pt}
		\setlength\abovedisplayskip{3pt}
		\label{eq_pop_gradient_split}
		A_1&:=\langle\mathbb{E}\hat{\nabla}f(\mathbf{P}),\Delta\rangle=\langle\mathbf{PP}^T-\mathbf{G}^{\star},\Delta\mathbf{P}^T+\mathbf{P}\Delta^T\rangle\\
		&=\langle\Delta\bar{\mathbf{P}}^{\star T}+\bar{\mathbf{P}}^{\star}\Delta^T+\Delta\Delta^T,\Delta\bar{\mathbf{P}}^{\star T}+\bar{\mathbf{P}}^{\star}\Delta^T+2\Delta\Delta^T\rangle\nonumber\\
		&\overset{(i)}{\geq} \frac{5}{8}\overbrace{\Vert\Delta\bar{\mathbf{P}}^{\star T}+\bar{\mathbf{P}}^{\star}\Delta^T\Vert_F^2}^{a^2}-4\overbrace{\Vert\Delta\Delta^T\Vert_F^2}^{b^2}\nonumber\\
		\label{eq_pop_gradient_lowerbound}
		&\overset{(ii)}{\geq} (\frac{5}{4}\sigma_r^{\star}-4\Vert\Delta\Vert_F^2)\Vert\Delta\Vert_F^2,
	\end{align}
	where $(i)$ follows from opening up the inner product and using Cauchy-Schwarz, then followed by elementary inequality
	\begin{equation*}
		\setlength\belowdisplayskip{3pt}
		\setlength\abovedisplayskip{3pt}
		a^2+2b^2-3ab\geq (1-\frac{3\gamma^2}{2})a^2+(2-\frac{3}{2\gamma^2})b^2,
	\end{equation*}
	with $\gamma=\frac{1}{2}$. $(ii)$ dues to the fact that $\Vert\Delta\Delta^T\Vert_F\leq\Vert\Delta\Vert_F^2$, and $\Delta^T\bar{\mathbf{P}}^{\star}=\bar{\mathbf{P}}^{\star T}\Delta$, i.e.,
	\begin{align*}
		\setlength\belowdisplayskip{3pt}
		\setlength\abovedisplayskip{3pt}
		\Vert\Delta\bar{\mathbf{P}}^{{\star}T}+\bar{\mathbf{P}}^{\star}\Delta^T\Vert_F^2&=2\Vert\bar{\mathbf{P}}^{\star}\Delta^T\Vert_F^2+\mathrm{tr}(\bar{\mathbf{P}}^{\star}\bar{\mathbf{P}}^{\star T}\Delta\Delta^T)\\
		&\geq2\Vert\bar{\mathbf{P}}^{\star}\Delta^T\Vert_F^2\geq 2\sigma_r^{\star}\Vert\Delta\Vert_F^2.
	\end{align*} 
	We next upper bound the distortion between sampled gradient and its population version. For notation simplicity, let $\mathcal{H}_{\Omega}:=\frac{1}{p}\mathcal{R}_{\Omega}-\mathcal{I}$ for the moment. For any $\mathbf{P}$, we have
	\begin{align}
		\setlength\belowdisplayskip{3pt}
		\setlength\abovedisplayskip{3pt}
		&A_2:=|\langle \hat{\nabla}f(\mathbf{P})- \mathbb{E}\hat{\nabla}f(\mathbf{P}),\Delta\rangle|\nonumber\\
		&=|\langle \mathcal{H}_{\Omega}(\Delta\bar{\mathbf{P}}^{{\star}T}+\bar{\mathbf{P}}^{\star}\Delta^T+\Delta\Delta^T),  \Delta\bar{\mathbf{P}}^{{\star}T}+\bar{\mathbf{P}}^{\star}\Delta^T+2\Delta\Delta^T\rangle|\nonumber\\
		&\leq \underbrace{|\langle \mathcal{H}_{\Omega}(\Delta\bar{\mathbf{P}}^{{\star}T}+\bar{\mathbf{P}}^{\star}\Delta^T), \Delta\bar{\mathbf{P}}^{{\star}T}+\bar{\mathbf{P}}^{\star}\Delta^T\rangle|}_{L_1}\nonumber\\
		&+\underbrace{2|\langle \mathcal{H}_{\Omega}(\Delta\bar{\mathbf{P}}^{{\star}T}+\bar{\mathbf{P}}^{\star}\Delta^T),\Delta\Delta^T \rangle|}_{L_{21}}+\underbrace{2|\langle \mathcal{H}_{\Omega}(\Delta\Delta^T),\Delta\Delta^T \rangle|}_{L_{22}}\nonumber\\
		\label{eq_Sample_to_Exp_Distor_4term}
		&+\underbrace{|\langle\mathcal{H}_{\Omega}^*(\Delta\bar{\mathbf{P}}^{{\star}T}+\bar{\mathbf{P}}^{\star}\Delta^T), \Delta\Delta^T\rangle|}_{L_3}.
	\end{align}
	\eqref{eq_Sample_to_Exp_Distor_4term} is quite different from existing literature on non-convex matrix completion (e.g., \cite{SL15NonCVXFR}\cite{ZhengLaffertyNonCVXFR}\cite{LiCuiPGD_SpectralCSHankel}\cite{YCLiSunSNLRR}\cite{MaoChenBlindSR_PGD}\cite{CaiWangSpecCS_PGD}). This asymmetry structure and heavy diagonal basis $\boldsymbol{\omega}_{\boldsymbol{\alpha}}$ make driving universal bounds on $L_{21}$, $L_{22}$, $L_3$ challenging. Similar predicament will occur in the restricted smoothness part, causing the step size to be dependent on $\mu$, $r$.
	
	\textit{Bound on $L_1$:} Recall that $\Delta\bar{\mathbf{P}}^{{\star}T}+\bar{\mathbf{P}}^{\star}\Delta^T\in\mathbb{T}$, and 
	\begin{equation*}
		\setlength\belowdisplayskip{3pt}
		\setlength\abovedisplayskip{3pt}
		\bigg\Vert \frac{1}{p}\mathcal{P}_{\mathbb{T}}\mathcal{R}_{\Omega}\mathcal{P}_{\mathbb{T}}-\mathcal{P}_{\mathbb{T}} \bigg\Vert\leq\frac{\varepsilon_{L_1}}{c\kappa}<1,
	\end{equation*}
	if $p\gtrsim C_\beta \frac{\kappa^2\mu^2 r^2\log n}{\varepsilon_{L_1}^2n}$ by Lemma \ref{lema_Romega_RIPL}. Thus for some $c>4$
	\begin{align}
		\setlength\belowdisplayskip{3pt}
		\setlength\abovedisplayskip{3pt}
		L_1=&\langle(\frac{1}{p}\mathcal{P}_{\mathbb{T}}\mathcal{R}_{\Omega}\mathcal{P}_{\mathbb{T}}-\mathcal{P}_{\mathbb{T}})(\Delta\bar{\mathbf{P}}^{{\star}T}+\bar{\mathbf{P}}^{\star}\Delta^T),\nonumber\\
		&\Delta\bar{\mathbf{P}}^{{\star}T}+\bar{\mathbf{P}}^{\star}\Delta^T\rangle\nonumber\leq\frac{\varepsilon_{L_1}}{c\kappa}\Vert\Delta\bar{\mathbf{P}}^{{\star}T}+\bar{\mathbf{P}}^{\star}\Delta^T\Vert_F^2\nonumber\\
		\label{eq_upper_bound_on_L1}
		&\leq \frac{4\sigma_1^{\star}\varepsilon_{L_1}}{c\kappa}\Vert\Delta\Vert_F^2\lesssim\varepsilon_{L_1}\sigma_r^{\star}\Vert\Delta\Vert_F^2.
	\end{align}

	\textit{Bound on $L_{21}$ and $L_{22}$:} We use the following sharp bound to handle $L_{21}$ and $L_{22}$, its proof is deferred to App. \ref{proof_of_lema_B1_ROAB}.
	\begin{lemma}
		\label{lema_sharp_bound_for_L21L22}
		Uniformly for all matrix $\mathbf{A}$, $\mathbf{B}$, $\mathbf{C}$, $\mathbf{D}\in\mathbb{R}^{n\times r}$, we have
		\begin{align*}
			\setlength\belowdisplayskip{3pt}
			\setlength\abovedisplayskip{3pt}
			\bigg|\bigg\langle (\frac{1}{p}\mathcal{R}_{\Omega}-\mathcal{I})(\mathbf{AB}^T)&,\mathbf{CD}^T \bigg\rangle\bigg|\\
			&\leq 2c_g\sqrt{\frac{n}{p}}\Vert\mathbf{A}\Vert_{2,\infty}\Vert\mathbf{B}\Vert_{2,\infty}\Vert\mathbf{CD}^T\Vert_F,
		\end{align*}
		holds for some constant $c_g>0$ with probability at least $1-n^{1-\beta}$ given $p\gtrsim \frac{C_\beta\log n}{n}$.
	\end{lemma}
	Therefore, by the symmetric structure of $\Delta\bar{\mathbf{P}}^{{\star}T}+\bar{\mathbf{P}}^{\star}\Delta^T$, we have
	\begin{align}
		\setlength\belowdisplayskip{3pt}
		\setlength\abovedisplayskip{3pt}
		L_{21}&\leq 8c_g\sqrt{\frac{n}{p}}\Vert\Delta\Vert_{2,\infty}\Vert\mathbf{P}^{\star}\Vert_{2,\infty}\Vert\Delta\Vert_F^2\nonumber\\
		\label{eq_unif_bound_on_I21}
		&\overset{(i)}{\leq} 8c_g\sqrt{\frac{c_I\mu^2 r^2\sigma_1^{\star 2}}{np}}\Vert\Delta\Vert_F^2\overset{(ii)}{\lesssim}\varepsilon_{L_{21}}\sigma_r^{\star}\Vert\Delta\Vert_F^2,
	\end{align}
	where $(i)$ uses the incoherence assumption on the ground truth and attractive region, i.e., $\Vert\Delta\Vert_{2,\infty}\leq\sqrt{\frac{c_I\mu r\sigma_1^{\star}}{n}}$. $(ii)$ by forcing $p\gtrsim C_{\beta}(\kappa^2 \mu^2 r^2)/(n\varepsilon_{L_{21}}^2)$. The same argument works for $L_{22}$, that is for $p\gtrsim C_{\beta}(\kappa^2 \mu^2 r^2)/(n\varepsilon_{L_{22}}^2)$, we have
	\begin{align}
		\setlength\belowdisplayskip{3pt}
		\setlength\abovedisplayskip{3pt}
		L_{22}&\leq 8c_g\sqrt{\frac{n}{p}}\Vert\Delta\Vert_{2,\infty}^2\Vert\Delta\Vert_F^2\nonumber\\
		\label{eq_unif_bound_on_I22}
		&\leq 8c_g\sqrt{\frac{c_I^2\mu^2 r^2 \sigma_1^{\star 2}}{np}}\Vert\Delta\Vert_F^2\lesssim\varepsilon_{L_{22}}\sigma_r^{\star}\Vert\Delta\Vert_F^2.
	\end{align}

	\textit{Bound on $L_3$:} This is slightly larger, and the ``uniform" part crucially relies on the fact that $\mathbf{P}^{\star}$ is a fixed matrix. It seems hard to reinforce Lemma \ref{lema_nonUni_bound_for_L3} to hold uniformly over all rank $r$ incoherent matrix $\mathbf{B}$. The proof is referred to App. \ref{proof_of_lema_B2_ROadjAB}.
	\begin{lemma}
		\label{lema_nonUni_bound_for_L3}
		Uniformly for all matrix $\mathbf{A}$, $\mathbf{C}$, $\mathbf{D}\in\mathbb{R}^{n\times r}$, $\boldsymbol{R}\in O(r)$, and some fix matrix $\mathbf{B}\in\mathbb{R}^{n\times r}$ independent with $\Omega$, which satisfy $\mathbf{A}^T\mathbf{1}=\mathbf{0}$, $\mathbf{B}^T\mathbf{1}=\mathbf{0}$, we have
		\begin{align*}
			\setlength\belowdisplayskip{3pt}
			\setlength\abovedisplayskip{3pt}
			\bigg|\bigg\langle (\frac{1}{p}\mathcal{R}_{\Omega}^*&-\mathcal{I})(\mathbf{A\boldsymbol{R}}^T\mathbf{B}^T),\mathbf{CD}^T \bigg\rangle\bigg|\\
			& \leq \bigg(
			C\sqrt{\frac{\beta n \log n}{p}}\Vert\mathbf{A}\Vert_{2,\infty}\Vert\mathbf{B}\Vert_{2,\infty}\bigg)\Vert\mathbf{CD}^T\Vert_F,
		\end{align*}
		holds for some constant $C>0$ with probability at least $1-n^{1-\beta}-n^{2-\beta}$ given $p\gtrsim \frac{C_\beta\log n}{n}$.
	\end{lemma}
	Thus, by using the symmetric structure of $(\frac{1}{p}\mathcal{R}_{\Omega}-\mathcal{I})(\Delta\Delta^T)$, we have
	\begin{align}
		\setlength\belowdisplayskip{3pt}
		\setlength\abovedisplayskip{3pt}
		L_3&\leq 2\bigg| \langle(\frac{1}{p}\mathcal{R}_{\Omega}^*-\mathcal{I})(\Delta\boldsymbol{\psi}^{\star T}\mathbf{P}^{\star T}),\Delta\Delta^T\rangle \bigg| \nonumber \\ 
		&\leq 2C\sqrt{\frac{\beta n\log n}{p}}\Vert\Delta\Vert_{2,\infty}\Vert\mathbf{P}^{\star}\Vert_{2,\infty}\Vert\Delta\Vert_F^2\nonumber\\
		\label{eq_non_unif_bound_on_I3}
		&\leq 2C\sqrt{\frac{c_I \mu^2 r^2 \sigma_1^{\star 2}\log n}{np}}\overset{(i)}{\lesssim}\varepsilon_{L_{3}}\sigma_r^{\star}\Vert\Delta\Vert_F^2,
	\end{align}
	holds uniformly over all $\mathbf{P}\in\mathcal{B}$, where $(i)$ by setting $p\gtrsim C_{\beta} \frac{\kappa^2 \mu^2 r^2 \log n}{\varepsilon_{L_{3}}^2 n}$.
	
	Substituting \eqref{eq_upper_bound_on_L1}, \eqref{eq_unif_bound_on_I21}, \eqref{eq_unif_bound_on_I22}, and \eqref{eq_non_unif_bound_on_I3} into \eqref{eq_Sample_to_Exp_Distor_4term}, and set $\varepsilon_{L_1}=\varepsilon_{L_{21}}=\varepsilon_{L_{22}}=\varepsilon_{L_3}=0.1$, it gives
	\begin{equation}
		\setlength\belowdisplayskip{3pt}
		\setlength\abovedisplayskip{3pt}
		\label{eq_upperbound_on_A2}
		A_2\leq \frac{2}{5}\sigma_r^{\star}\Vert\Delta\Vert_F^2,
	\end{equation}
	holds for probability at least $1-3n^{1-\beta}-n^{2-\beta}$ inside the attractive region, given $C_{\beta}$ large enough. By combining \eqref{eq_upperbound_on_A2} and \eqref{eq_pop_gradient_lowerbound}, we can decide the upper bound on $\Vert\Delta\Vert_F$
	\begin{align}
		\setlength\belowdisplayskip{3pt}
		\setlength\abovedisplayskip{3pt}
		\langle \hat{\nabla}&f(\mathbf{P}),\Delta \rangle\geq \langle \mathbb{E}\hat{\nabla}f(\mathbf{P}),\Delta\rangle-\frac{2}{5}\sigma_r^{\star}\Vert\Delta\Vert_F^2\nonumber\\
		\label{eq_restriced_CVX_finalbound}
		&\geq(\frac{5}{4}\sigma_r^{\star}-\frac{2}{5}\sigma_r^{\star}-4\Vert\Delta\Vert_F^2)\Vert\Delta\Vert_F^2\overset{(i)}{\geq}\frac{1}{2}\sigma_r^{\star}\Vert\Delta\Vert_F^2,
	\end{align}
	where $(i)$ from forcing $\Vert\Delta\Vert_F^2\leq\frac{7}{80}\sigma_r^{\star}$, concluding the proof.
	
	 \qed
	\subsection{Restricted Smoothness: \eqref{eq_re_RE_Smooth}}
	\label{Appdi_B2_Bound}
	Notice that $\Vert\hat{\nabla} f(\mathbf{P})\Vert_F^2=|\sup_{\Vert\mathbf{Z}\Vert_F^2=1}\langle\hat{\nabla} f(\mathbf{P}),\mathbf{Z}\rangle|^2$, and $A_3:=|\langle\hat{\nabla} f(\mathbf{P}),\mathbf{Z}\rangle|^2$ can be separated as in \eqref{eq_Sample_to_Exp_Distor_4term}. Define $\mathbf{Z}_{\Delta}:=\mathbf{Z}\Delta^T+\Delta\mathbf{Z}^T$, and $\mathbf{Z}_{P^{\star}}=\mathbf{Z}\bar{\mathbf{P}}^{\star T}+\bar{\mathbf{P}}^{\star}\mathbf{Z}^T\in \mathbb{T}$, we have
	\begin{align}
		\setlength\belowdisplayskip{3pt}
		\setlength\abovedisplayskip{3pt}
		A_3&=\left|\bigg\langle\frac{1}{p}\mathcal{R}_{\Omega}(\Delta\bar{\mathbf{P}}^{{\star}T}+\bar{\mathbf{P}}^{\star}\Delta^T+\Delta\Delta^T),\mathbf{Z}_{P^{\star}}+\mathbf{Z}_{\Delta}\bigg\rangle\right|^2\nonumber\\
		&\leq \bigg| \bigg\langle \frac{1}{p}\mathcal{R}_{\Omega}(\Delta\bar{\mathbf{P}}^{{\star}T}+\bar{\mathbf{P}}^{\star}\Delta^T),\mathbf{Z}_{P^{\star}} \bigg\rangle\nonumber \\
		&+ \bigg\langle \frac{1}{p}\mathcal{R}_{\Omega}(\Delta\bar{\mathbf{P}}^{{\star}T}+\bar{\mathbf{P}}^{\star}\Delta^T),\mathbf{Z}_{\Delta} \bigg\rangle\nonumber\\
		&+\bigg\langle \frac{1}{p}\mathcal{R}_{\Omega}(\Delta\Delta^T), \mathbf{Z}_{\Delta}\bigg\rangle+\bigg\langle \frac{1}{p}\mathcal{R}_{\Omega}(\Delta\Delta^T), \mathbf{Z}_{P^{\star}}\bigg\rangle\bigg|^2\nonumber\\
		\label{eq_A3_spliting}
		&:=|R_1+R_2+R_3+R_4|^2\nonumber\\
		&\leq ||R_1|+|R_2|+|R_3|+|R_4||^2.
	\end{align}

	\textit{Bound on $R_1$, $R_2$, and $R_3$:} The upper bounds on the first three terms are meanwhile straightforward, thus we process them quickly. When $p\gtrsim\frac{C_{\beta} \mu^2 r^2 \log n}{\varepsilon_{R_1}^2 n}$ and inside the RIC, since $\Vert\mathbf{Z}\Vert_F=1$, we have
	\begin{align}
		\setlength\belowdisplayskip{3pt}
		\setlength\abovedisplayskip{3pt}
		&|R_1|=\bigg|\bigg\langle (\frac{1}{p}\mathcal{P}_{\mathbb{T}}\mathcal{R}_{\Omega}\mathcal{P}_{\mathbb{T}}-\mathcal{P}_{\mathbb{T}})(\Delta\bar{\mathbf{P}}^{{\star}T}+\bar{\mathbf{P}}^{\star}\Delta^T), \mathbf{Z}_{P^{\star}}\bigg\rangle\nonumber\\
		&+\langle\Delta\bar{\mathbf{P}}^{{\star}T}+\bar{\mathbf{P}}^{\star}\Delta^T,\mathbf{Z}_{P^{\star}}\rangle\bigg|\nonumber\\
		\label{eq_upper_bound_R1}
		&\overset{(i)}{\leq}4(1+\varepsilon_{R_1})\sigma_1^{\star}\Vert\Delta\Vert_F\Vert\mathbf{Z}\Vert_F\leq \frac{22}{5}\sigma_1^{\star}\Vert\Delta\Vert_F,\\
		&|R_2|=\bigg|\bigg\langle (\frac{1}{p}\mathcal{R}_{\Omega}-\mathcal{I})(\Delta\bar{\mathbf{P}}^{{\star}T}+\bar{\mathbf{P}}^{\star}\Delta^T), \mathbf{Z}_{\Delta}\bigg\rangle\nonumber\\
		&+\langle\Delta\bar{\mathbf{P}}^{{\star}T}+\bar{\mathbf{P}}^{\star}\Delta^T,\mathbf{Z}_{\Delta}\rangle\bigg|\nonumber\\
		&\overset{(ii)}{\leq} (6c_g\sqrt{\frac{n}{p}}\Vert\Delta\Vert_{2,\infty}\Vert\mathbf{P}^{\star}\Vert_{2,\infty}+4\sigma_1^{\star1/2}\Vert\Delta\Vert_F)\Vert\Delta\Vert_F\Vert\mathbf{Z}\Vert_F\nonumber\\
		\label{eq_upper_bound_R2}
		&\overset{(iii)}{\lesssim} (\sqrt{\frac{c_I \mu^2 r^2\sigma_1^{\star 2}}{np}}+ \frac{6}{5}\sigma_1^{\star}) \Vert\Delta\Vert_F\leq\frac{7}{5}\sigma_1^{\star}\Vert\Delta\Vert_F,\\
		&|R_3|=\bigg|\bigg\langle (\frac{1}{p}\mathcal{R}_{\Omega}-\mathcal{I})(\Delta\Delta^T), \mathbf{Z}_{\Delta}\bigg\rangle+\langle\Delta\Delta^T,\mathbf{Z}_{\Delta}\rangle\bigg|\nonumber\\
		&\overset{(iv)}{\leq} (3c_g\sqrt{\frac{n}{p}}\Vert\Delta\Vert_{2,\infty}^2+2\Vert\Delta\Vert_F^2)\Vert\Delta\Vert_F\Vert\mathbf{Z}\Vert_F\nonumber\\
		\label{eq_upper_bound_R3}
		&\overset{(v)}{\lesssim} (\sqrt{\frac{c_I^2 \mu^2 r^2\sigma_1^{\star 2}}{np}}+\frac{7}{40}\sigma_1^{\star})\Vert\Delta\Vert_F\leq\frac{1}{5}\sigma_1^{\star}\Vert\Delta\Vert_F,
	\end{align}
	where $(i)$ from Lemma \ref{lema_Romega_RIPL} and set $\varepsilon_{R_1}=0.1$, $(ii)$, $(iv)$ from Lemma \ref{lema_sharp_bound_for_L21L22}, $(iii)$, $(v)$ by the assumption on $\mathbf{P}^{\star}$ and the RIC.   
	
	\textit{Bound on $R_4$:} This is the largest one since it does not follow from Lemma \ref{lema_nonUni_bound_for_L3}. We are inclined to believe that it is hard to substantively improve the estimate below, i.e., to get rid of $\mu$, $r$. Notice that w.l.o.g., we can assume $\mathbf{Z}^T\mathbf{1}=\mathbf{0}$ as $\mathbf{Z}_{P^{\star}}\in\mathbb{T}$. Therefore $\mathbf{J}\mathbf{Z}_{P^{\star}}\mathbf{J}=\mathbf{Z}_{P^{\star}}$, it gives
	\begin{align}
		\setlength\belowdisplayskip{3pt}
		\setlength\abovedisplayskip{3pt}
		|R_4|&=\bigg| \bigg\langle\frac{1}{p}g^+\mathcal{P}_{\Omega}g(\Delta\Delta^T), \mathbf{Z}_{P^{\star}} \bigg\rangle\bigg|\nonumber\\
		&=\bigg| \bigg\langle \frac{1}{p}\mathcal{P}_{\Omega}g(\Delta\Delta^T),\mathcal{P}_{\Omega}g^+(\mathbf{Z}_{P^{\star}}) \bigg\rangle\bigg|\nonumber\\
		&\overset{(i)}{\leq}\frac{1}{2}\sqrt{\frac{1}{p}\Vert\mathcal{Q}_{\Omega}(\Delta\Delta^T)\Vert_F^2}\sqrt{\frac{1}{p}\Vert\mathcal{P}_{\Omega}\mathcal{P}_{\mathbb{T}}(\mathbf{Z}_{P^{\star}})\Vert_F^2}\nonumber\\
		\label{eq_upperbound_R_4}
		&\overset{(ii)}{\leq}\frac{63\sigma_1^{\star}\sqrt{c_I\mu r}}{20}\Vert\Delta\Vert_F,
	\end{align}
	where $(i)$ uses the structure of $g^+$ and followed by Cauchy-Schwarz. And $(ii)$ dues to combining
	\begin{align*}
		\setlength\belowdisplayskip{3pt}
		\setlength\abovedisplayskip{3pt}
		&\frac{1}{p}\Vert\mathcal{Q}_{\Omega}(\Delta\Delta^T)\Vert_F^2\\
		&\overset{(a)}{\leq} (8c_I\mu r \sigma_1^{\star} +\sqrt{\frac{c\beta c_I^2 \mu^2 r^2 \sigma_1^{\star 2}\log n}{np}} + \frac{7}{10}\sigma_r^{\star})\Vert\Delta\Vert_F^2\\
		&\overset{(b)}{\leq} 9c_I\mu r \sigma_1^{\star} \Vert\Delta\Vert_F^2,
	\end{align*}
	and $\sqrt{\frac{1}{p}\Vert\mathcal{P}_{\Omega}\mathcal{P}_{\mathbb{T}}(\mathbf{Z}_{P^{\star}})\Vert_F^2}\overset{(c)}{\leq} \sqrt{\frac{11}{10}\Vert\mathbf{Z}_{P^{\star}}\Vert_F^2}\leq \frac{\sqrt{110\sigma_1^{\star}}}{5}\Vert\mathbf{Z}\Vert_F$.
	Where $(a)$ from Lemma \ref{eq_NromalSpace_tight_bound} and then use the RIC assumption on $\mathcal{B}$, $(b)$ by forcing $p\gtrsim\frac{C_{\beta} \mu^2 r^2 \log n}{n}$, and $(c)$ from Lemma \ref{lema_Pomega_RIPL_Cand} with $\epsilon=0.1$ and $p\gtrsim \frac{C_{\beta}\mu r\log n}{\epsilon^2 n}$. 
	
	By substituting \eqref{eq_upper_bound_R1}, \eqref{eq_upper_bound_R2}, \eqref{eq_upper_bound_R3}, and \eqref{eq_upperbound_R_4} into \eqref{eq_A3_spliting}, it gives
	\begin{equation}
		\setlength\belowdisplayskip{3pt}
		\setlength\abovedisplayskip{3pt}
		\label{eq_final_bound_on_A3}
		\Vert\hat{\nabla} f(\mathbf{P})\Vert_F^2\leq 84\sigma_1^{\star 2}c_I\mu r\Vert\Delta\Vert_F^2,
	\end{equation}
	hold with probability at least $1-cn^{1-\beta}$ in the RIC $\mathcal{B}$, provided with $p\gtrsim (C_{\beta}\mu^2 r^2\log n)/n$ and $C_{\beta}$ large enough, thus concluding the proof. \qed
	\section{Supporting Lemmas}
	\begin{lemma}[\cite{TroppUserFriendMCI}, Thm. 1.6]
		\label{eq_matrix_Bernstein}
		If a finite sequence $\{\mathbf{S}_k\}$ of independent, random matrices of dimension $\mathbb{R}^{n_1\times n_2}$ that satisfy
		\begin{equation*}
			\setlength\belowdisplayskip{3pt}
			\setlength\abovedisplayskip{3pt}
			\mathbb{E}\mathbf{S}_k=\mathbf{0},\,\Vert\mathbf{S}_k\Vert\leq B \,\,\text{almost surely.}
		\end{equation*}
		Let the norm of the total variance be
		\begin{equation*}
			\sigma^2:=\max\bigg\{\bigg\Vert\sum_k\mathbb{E}\mathbf{S}_k\mathbf{S}_k^T\bigg\Vert,\bigg\Vert\sum_k\mathbb{E}\mathbf{S}_k^T\mathbf{S}_k\bigg\Vert\bigg\}.
		\end{equation*}
		Then for all $t\geq0$ we have
		\begin{align*}
			\setlength\belowdisplayskip{3pt}
			\setlength\abovedisplayskip{3pt}
			\mathbb{P}&\left\{\bigg\Vert\sum_k \mathbf{S}_k\bigg\Vert\geq t\right\}\leq (n_1+n_2) \exp(\frac{-t^2}{2\sigma^2+2Bt/3}).
		\end{align*}
	\end{lemma}
	\begin{lemma}[\cite{BandeiraHandelSharpRGL}, Coro. 3.12]
		\label{random_graph_lemma}
		The random graph lemma, see also\textnormal{\cite[Lemma 7.1]{KeshavanOptSpace}}. If $p\geq \frac{C_g\beta\log n}{n}$ for some $C_g>0$, $\beta>1$, $c_g>1$
		\begin{equation*}
			\setlength\belowdisplayskip{3pt}
			\setlength\abovedisplayskip{3pt}
			\bigg\Vert \frac{1}{p}\mathcal{P}_{\Omega}(\mathbf{H}_{\mathbf{1}})-\mathbf{H}_{\mathbf{1}} \bigg\Vert \leq c_g \sqrt{\frac{n}{p}},
		\end{equation*}
	holds with probability at least $1-n^{1-\beta}$, where $\mathbf{H}_{\mathbf{1}} = \mathbf{11}^T-\mathbf{I}_n$.
	\end{lemma}
	\begin{lemma}[\cite{SmithCaiTasissaRieEDMC2}, Thm. 5.4]
		\label{lema_Romega_RIPL}
		The local RIP of $\mathcal{R}_{\Omega}$.
		If $\Vert\mathbf{U}^{\star}\Vert_{2,\infty}^2\leq \frac{\mu r}{n}$, then for $\epsilon\geq \sqrt{\frac{C_r\beta (\mu r)^2\log n}{np}}$
		\begin{equation*}
			\setlength\belowdisplayskip{3pt}
			\setlength\abovedisplayskip{3pt}
			\bigg\Vert\frac{1}{p}\mathcal{P}_{\mathbb{T}}\mathcal{R}_{\Omega}\mathcal{P}_{\mathbb{T}}-\mathcal{P}_{\mathbb{T}}\bigg\Vert \leq \epsilon< 1,
		\end{equation*}
		holds with probability at least $1-n^{1-\beta}$ as soon as $p\geq C_r\beta(\mu r)^{2}\log n/n$ for sufficient large constant $C_r$ and $\beta>1$.
	\end{lemma}
	\begin{lemma}[\cite{LRMC_Can1}, Thm. 4.1]
		\label{lema_Pomega_RIPL_Cand}
		The local RIP of $\mathcal{P}_{\Omega}$. If $\Vert\mathbf{U}^{\star}\Vert_{2,\infty}^2\leq \frac{\mu r}{n}$, then for $\epsilon\geq \sqrt{\frac{C_r\beta\mu r\log n}{np}}$
		\begin{equation*}
			\setlength\belowdisplayskip{3pt}
			\setlength\abovedisplayskip{3pt}
			\bigg\Vert\frac{1}{p}\mathcal{P}_{\mathbb{T}}\mathcal{P}_{\Omega}\mathcal{P}_{\mathbb{T}}-\mathcal{P}_{\mathbb{T}}\bigg\Vert \leq \epsilon< 1,
		\end{equation*}
		holds with probability at least $1-n^{1-\beta}$ as soon as $p\geq C_r\beta\mu r\log n/n$ for sufficient large constant $C_r$ and $\beta>1$.
	\end{lemma}
	\begin{lemma}
		\label{eq_NromalSpace_tight_bound}
		If $p>C\frac{\beta\log n}{n}$, then uniformly for all matrices $\Delta\in\mathbb{R}^{n\times r}$,\vspace{-12pt}
		
		\begin{small}
			\begin{align*}
				\setlength\belowdisplayskip{3pt}
				\setlength\abovedisplayskip{3pt}
				\frac{1}{p}\Vert\mathcal{Q}_{\Omega}(\Delta\Delta^T)\Vert_F^2
				\leq \left[(8n+\sqrt{\frac{c\beta n\log n}{p}})\Vert\Delta\Vert_{2,\infty}^2+8\Vert\Delta\Vert_F^2\right]\Vert\Delta\Vert_F^2,
			\end{align*}
		\end{small}
		
		\vspace{-4pt}\hspace{-8pt}holds with probability at least $1-n^{1-\beta}$. This is in fact an unpublished result in the extended version of \textnormal{\cite{YCLiSunSNLRR}}, but for completeness we reprove it here in App. \ref{proof_of_lema_B5_NromalSpac}.
	\end{lemma}
	\begin{lemma}[\cite{YCLiSunEDMC_Spec_init}, Thm. I.1]
		\label{lema_OSMDS_Spec_init}
		See also \textnormal{\cite[Lemma 5.6]{SmithCaiTasissaRieEDMC2}}. Under the set up of Theorem \ref{thm_main_global_contraction}, for some $0<\epsilon\leq 1$, we have
		\begin{equation}
			\setlength\belowdisplayskip{3pt}
			\setlength\abovedisplayskip{3pt}
			\label{eq_theorem_main_approx_error}
			\bigg\Vert\mathcal{T}_r\left[ -\frac{1}{2p} \mathbf{J} (\mathcal{P}_{\Omega}\mathbf{D}^{\star}) \mathbf{J}\right]-\mathbf{G}^{\star}\bigg\Vert\leq \epsilon \Vert\mathbf{G}^{\star}\Vert,
		\end{equation}
		holds with probability at least $1-n^{1-\beta}$ as soon as $p\geq C_s\beta(\mu r)^2\log n/(\epsilon^2n)$.
	\end{lemma}
	\begin{lemma}[\cite{SL15NonCVXFR}, Prop. B.4]
		\label{lema_AB_spec_F_norm_to_SingVal}
		For any matrix $\mathbf{E}\in\mathbb{R}^{n\times r}$, $\mathbf{F}\in\mathbb{R}^{r\times n}$, $r\leq n$, we have
		\begin{equation*}
			\setlength\belowdisplayskip{3pt}
			\setlength\abovedisplayskip{3pt}
			\sigma_{\min}(\mathbf{E})\Vert\mathbf{F}\Vert_F\leq\Vert\mathbf{EF}\Vert_F\leq \Vert\mathbf{E}\Vert\Vert\mathbf{F}\Vert_F.
		\end{equation*}
	\end{lemma}
	\section{Proof of Auxiliary Lemmas}
	\label{Appdi_C_AuxiliaryLemma}
	The proof of Lemma \ref{lema_sharp_bound_for_L21L22} and \ref{lema_nonUni_bound_for_L3} is partly inspired by \cite[Lemma 8]{ChenJiLiModelFreeMC}\cite[Lemma 22]{DingChenLOOPrimalDual}, and \cite[Thm. 4.1]{pmlr-v32-bhojanapalli14}, respectively.
	\subsection{Proof of Lemma \ref{lema_sharp_bound_for_L21L22}}
	\label{proof_of_lema_B1_ROAB}
	Notice that 
	\begin{equation*}
		\setlength\belowdisplayskip{3pt}
		\setlength\abovedisplayskip{3pt}
		\bigg|\bigg\langle (\frac{1}{p}\mathcal{R}_{\Omega}-\mathcal{I})(\mathbf{AB}^T),\mathbf{CD}^T \bigg\rangle\bigg|\leq\bigg\Vert(\frac{1}{p}\mathcal{R}_{\Omega}-\mathcal{I})(\mathbf{AB}^T)\bigg\Vert\Vert\mathbf{CD}^T\Vert_F,
	\end{equation*}
	and
	\begin{align*}
			\setlength\belowdisplayskip{3pt}
			\setlength\abovedisplayskip{3pt}
			\bigg\Vert(\frac{1}{p}\mathcal{R}_{\Omega}-\mathcal{I})&(\mathbf{AB}^T)\bigg\Vert=\bigg\Vert\frac{1}{p}g^+\mathcal{P}_{\Omega}g(\mathbf{AB}^T)-g^+g(\mathbf{AB}^T)\bigg\Vert\\
			&=\frac{1}{2}\Vert\mathbf{J}\Vert\bigg\Vert \frac{1}{p}\mathcal{P}_{\Omega}g(\mathbf{AB}^T)-g(\mathbf{AB}^T) \bigg\Vert\Vert\mathbf{J}\Vert\\
			&\leq \frac{1}{2}\bigg\Vert \frac{1}{p}\mathcal{P}_{\Omega}g(\mathbf{AB}^T)-g(\mathbf{AB}^T) \bigg\Vert=\frac{1}{2}I_1.
	\end{align*}
	Using the variational characterization of spectral norm, we have
	\begin{align*}
			\setlength\belowdisplayskip{3pt}
			\setlength\abovedisplayskip{3pt}
			I_1&:=\sup_{\mathbf{x},\mathbf{y}\in\mathbb{S}^{n-1}}\bigg\langle \frac{1}{p}\mathcal{P}_{\Omega}g(\mathbf{AB}^T)-g(\mathbf{AB}^T),\mathbf{xy}^T \bigg\rangle\\
			&\overset{(i)}{=}\sup_{\mathbf{x},\mathbf{y}\in\mathbb{S}^{n-1}}\bigg\langle \frac{1}{p}\mathcal{P}_{\Omega}(\mathbf{H}_{\mathbf{1}})-\mathbf{H}_{\mathbf{1}}, g(\mathbf{AB}^T)\circ \mathbf{xy}^T \bigg\rangle\\
			&\leq \bigg\Vert \frac{1}{p}\mathcal{P}_{\Omega}(\mathbf{H}_{\mathbf{1}})-\mathbf{H}_{\mathbf{1}} \bigg\Vert \sup_{\mathbf{x},\mathbf{y}\in\mathbb{S}^{n-1}}\Vert g(\mathbf{AB}^T)\circ \mathbf{xy}^T \Vert_*\\
			&\overset{(ii)}{\leq} c_g\sqrt{\frac{n}{p}} \sup_{\mathbf{x},\mathbf{y}\in\mathbb{S}^{n-1}}I_2(\mathbf{x},\mathbf{y}),
	\end{align*}
	where $(i)$ by noticing
	\begin{equation*}
		\setlength\belowdisplayskip{3pt}
		\setlength\abovedisplayskip{3pt}
		\frac{1}{p}\mathcal{P}_{\Omega}g(\mathbf{AB}^T)-g(\mathbf{AB}^T)=(\frac{1}{p}\mathcal{P}_{\Omega}(\mathbf{H}_{\mathbf{1}})-\mathbf{H}_{\mathbf{1}})\circ g(\mathbf{AB}^T)
	\end{equation*}
	and then using the identity $\mathrm{tr}(\mathbf{A}^T(\mathbf{B}\circ\mathbf{C}))=\mathrm{tr}((\mathbf{A}^T\circ\mathbf{B}^T)\mathbf{C})$, and $(ii)$ from Lemma \ref{random_graph_lemma}.
	By \eqref{eq_Gram_to_EDM} and triangle inequality, $I_2$ reduces to 
	\begin{align*}
		\setlength\belowdisplayskip{3pt}
		\setlength\abovedisplayskip{3pt}
		I_2\leq &\Vert \mathrm{diag}(\mathbf{AB}^T)\mathbf{1}^T\circ \mathbf{xy}^T\Vert_* +  \Vert \mathbf{1}\mathrm{diag}(\mathbf{AB}^T)^T\circ \mathbf{xy}^T\Vert_* \\
		&+ 2\Vert \mathbf{AB}^T\circ \mathbf{xy}^T\Vert_* = I_{21}+ I_{22}+ 2I_{23}.
	\end{align*} 
	For $I_{21}$ and $I_{22}$, recall $\mathbf{A}=[\mathbf{A}_{1,.},\cdots,\mathbf{A}_{n,.}]^T$, $\mathbf{B}=[\mathbf{B}_{1,.},\cdots,\mathbf{B}_{n,.}]^T$. By the identity $(\mathbf{ab}^T)\circ(\mathbf{cd}^T)=(\mathbf{a}\circ\mathbf{c})(\mathbf{b}\circ\mathbf{d})^T$, we have
	\begin{align*}
		\setlength\belowdisplayskip{3pt}
		\setlength\abovedisplayskip{3pt}
		I_{21}&=\Vert (\mathrm{diag}(\mathbf{AB}^T)\circ\mathbf{x})(\mathbf{1}\circ\mathbf{y})^T\Vert_*\\
		&\leq \left\Vert
		\begin{bmatrix}
			\mathbf{x}_1 \mathbf{A}_{1,.}^T\mathbf{B}_{1,.}\\
			\vdots\\
			\mathbf{x}_n \mathbf{A}_{n,.}^T\mathbf{B}_{n,.}
		\end{bmatrix}
		\right\Vert_2\Vert\mathbf{y}\Vert_2\leq \sqrt{\sum_{i=1}^n\mathbf{x}_i^2(\mathbf{A}_{i,.}^T\mathbf{B}_{i,.})^2}\\
		&\leq\max_{i\in[n]}|\mathbf{A}_{i,.}^T\mathbf{B}_{i,.}|\cdot\Vert\mathbf{x}\Vert_2\leq\Vert\mathbf{AB}^T\Vert_{\infty}.
	\end{align*}
	Same bound holds for $I_{22}$. By \cite[Lemma 5]{pmlr-v48-lii16}, $I_{23}$ can be upper bounded by $\Vert\mathbf{A}\Vert_{2,\infty}\Vert\mathbf{B}\Vert_{2,\infty}$.
	Since $\Vert\mathbf{AB}^T\Vert_{\infty}\leq \Vert\mathbf{A}\Vert_{2,\infty}\Vert\mathbf{B}\Vert_{2,\infty}$, we have 
	\begin{equation*}
		\setlength\belowdisplayskip{3pt}
		\setlength\abovedisplayskip{3pt}
		\bigg\Vert(\frac{1}{p}\mathcal{R}_{\Omega}-\mathcal{I})(\mathbf{AB}^T)\bigg\Vert\leq2c_g\sqrt{\frac{n}{p}}\Vert\mathbf{A}\Vert_{2,\infty}\Vert\mathbf{B}\Vert_{2,\infty},
	\end{equation*}
	holds uniformly over all $\mathbf{A}$, $\mathbf{B}\in\mathbb{R}^{n\times r}$ with probability at least $1-n^{1-\beta}$ given $p\gtrsim \frac{\beta\log n}{n}$, concluding the proof. \qed
	\subsection{Proof of Lemma \ref{lema_nonUni_bound_for_L3}}
	\label{proof_of_lema_B2_ROadjAB}
	Unfortunately, same strategy as in App. \ref{proof_of_lema_B1_ROAB} does not applies to Lemma \ref{lema_nonUni_bound_for_L3}, due to the asymmetry structure of $\mathcal{R}_{\Omega}$. Uniform bound presents here dues to the fixed, incoherent matrix $\mathbf{B}$. 
	Let $\mathbf{X}=\mathbf{A}\boldsymbol{R}^T\mathbf{B}^T$, we need to control $T_0=\Vert(\frac{1}{p}\mathcal{R}_{\Omega}^*-\mathcal{I})\mathbf{X}\Vert$, where $\mathbf{X}=\mathbf{JXJ}$ always holds. Notice that by using the structure of $g^{*}(\mathbf{D}):=2(\mathrm{diag}(\mathbf{D}\mathbf{1})-\mathbf{D})$, $T_0$ can be reduced to 
	\begin{align*}
		\setlength\belowdisplayskip{3pt}
		\setlength\abovedisplayskip{3pt}
		T_0&=\bigg\Vert\frac{-1}{2p}g^*\mathcal{P}_{\Omega}\mathbf{X} - \mathbf{X}\bigg\Vert\\
		&=\bigg\Vert \frac{1}{p}\mathrm{diag}((\mathcal{P}_{\Omega}\mathbf{X})\mathbf{1}) + \frac{1}{p}\mathcal{P}_{\Omega}\mathbf{X} - \mathbf{X} \bigg\Vert\\
		&\overset{(i)}{\leq} \bigg\Vert(\frac{1}{p}\mathcal{P}_{\Omega}\mathbf{X})\mathbf{1}-\mathrm{diag}(\mathbf{X})\bigg\Vert_{\infty} + \bigg\Vert\frac{1}{p}\mathcal{P}_{\Omega}\mathbf{X} - \mathrm{Od}(\mathbf{X})\bigg\Vert\nonumber\\
		&\overset{(ii)}{\leq} \underbrace{\bigg\Vert(\frac{1}{p}\mathcal{P}_{\Omega}\mathbf{X}-\mathrm{Od}(\mathbf{X}))\mathbf{1}\bigg\Vert_{\infty}}_{T_1} + \underbrace{\bigg\Vert\frac{1}{p}\mathcal{P}_{\Omega}\mathbf{X} - \mathrm{Od}(\mathbf{X})\bigg\Vert}_{T_2},
	\end{align*}
	where $(i)$ use the fact that ${\Omega}$ is a hollow diagonal sample set, and $(ii)$ by noticing $\mathbf{X1}=\mathbf{0}$, i.e., $\mathbf{e}_i^T\mathrm{Od}(\mathbf{X})\mathbf{1}=\sum_{j\neq i}^{j\in [n]}\mathbf{X}_{ij} = \mathbf{X}_{ii}$ for $i\in [n]$.	
	\textit{Bound on $T_2$: }This follows from similar argument as in\cite[Lemma 22]{DingChenLOOPrimalDual}, thus holds uniformly over $\mathbf{A}$, $\mathbf{B}$, and $\boldsymbol{R}$. By using the same trick as in App. \ref{proof_of_lema_B1_ROAB}, one can show
	\begin{align}
		\setlength\belowdisplayskip{3pt}
		\setlength\abovedisplayskip{3pt}
		T_2 & =\sup_{\mathbf{x},\mathbf{y}\in\mathbb{S}^{n-1}}\bigg\langle\frac{1}{p}\mathcal{P}_{\Omega}(\mathbf{H}_{\mathbf{1}})-\mathbf{H}_{\mathbf{1}},\mathrm{Od}(\mathbf{A}\bar{\mathbf{B}}^T)\circ \mathbf{xy}^T\bigg\rangle\nonumber\\
		\label{eq_upperbound_T3_inprocess}
		&\leq c_g\sqrt{\frac{n}{p}}\bigg\Vert\mathrm{Od}(\mathbf{A}\bar{\mathbf{B}}^T)\circ \mathbf{xy}^T\bigg\Vert_{*}= c_g\sqrt{\frac{n}{p}} T_{3}.
	\end{align}
	where we set $\bar{\mathbf{B}}=\mathbf{B}\boldsymbol{R}$ for any $\boldsymbol{R}\in O(r)$. By subtracting and adding terms, it gives 
\begin{equation*}
	\setlength\belowdisplayskip{3pt}
	\setlength\abovedisplayskip{3pt}
	T_3\leq \underbrace{\Vert\mathbf{A}\bar{\mathbf{B}}^T\circ \mathbf{xy}^T\Vert_{*}}_{T_{31}}+\underbrace{\Vert(\mathbf{I}\circ\mathbf{A}\bar{\mathbf{B}}^T)\circ \mathbf{xy}^T\Vert_{*}}_{T_{32}},
\end{equation*}
and the same upper bound as $I_{23}$ applies to $T_{31}$. For $T_{32}$, first recall the column blocking structure, we have
\begin{align*}
	\setlength\belowdisplayskip{3pt}
	\setlength\abovedisplayskip{3pt}
	(\mathbf{I}\circ\mathbf{A}\bar{\mathbf{B}}^T)\circ \mathbf{xy}^T &= \sum_{i=1}^n \langle\sum_{k=1}^r \mathbf{A}_{.,k}\bar{\mathbf{B}}_{.,k}^T,(\mathbf{e}_i\mathbf{e}_i^T)\circ \mathbf{xy}^T\rangle \mathbf{e}_i\mathbf{e}_i^T\\
	& =\sum_{i=1}^n (\sum_{k=1}^r \mathbf{A}_{ik}\bar{\mathbf{B}}_{ik}\mathbf{x}_i\mathbf{y}_i)\mathbf{e}_i\mathbf{e}_i^T.
\end{align*}
Next, by using the identity $\Vert\mathbf{Y}\Vert_{*}=\mathrm{tr}(\sqrt{\mathbf{Y}^T\mathbf{Y}})$, it gives
\begin{align}
	\setlength\belowdisplayskip{3pt}
	\setlength\abovedisplayskip{3pt}
	T_{32} &= \sum_{i=1}^n\sqrt{ (\sum_{k=1}^r \mathbf{A}_{ik}\mathbf{x}_i\bar{\mathbf{B}}_{ik}\mathbf{y}_i)^2 }\nonumber\\
	&\overset{(i)}{\leq} \sum_{i=1}^n\sqrt{ (\sum_{k_1=1}^r\mathbf{A}_{ik_1}^2\mathbf{x}_i^2)}\sqrt{(\sum_{k_2=1}^r\bar{\mathbf{B}}_{ik_2}^2\mathbf{y}_i^2)}\nonumber\\
	&\overset{(ii)}{\leq} \sqrt{\sum_{i=1}^n\mathbf{x}_i^2\sum_{k_1=1}^r \mathbf{A}_{ik_1}^2}\sqrt{\sum_{j=1}^n\mathbf{y}_j^2\sum_{k_2=1}^r \bar{\mathbf{B}}_{jk_2}^2}\nonumber\\
	\label{eq_finalbound_diagnal_ABT_T_32}
	&\overset{(iii)}{\leq} \Vert\mathbf{A}\Vert_{2,\infty}\Vert\mathbf{B}\Vert_{2,\infty},
\end{align}
	where $(i)$, $(ii)$ use Cauchy-Schwarz, $(iii)$ from noticing $\sum_{k_1=1}^r \mathbf{A}_{ik_1}^2\leq\Vert\mathbf{A}\Vert_{2,\infty}^2$, $\forall i\in[n]$. Thus $T_3 \leq 2\Vert\mathbf{A}\Vert_{2,\infty}\Vert\mathbf{B}\Vert_{2,\infty}$. By substituting \eqref{eq_finalbound_diagnal_ABT_T_32} into \eqref{eq_upperbound_T3_inprocess}, we conclude the first part of the proof.

	\textit{Bound on $T_1$: }The uniform result does not apply to $\mathbf{B}$, yet the orthogonal matrix $\boldsymbol{R}$ can be automatically canceled out after some transformation, thus eliminating a covering argument over $O(r)$. W.l.o.g. we can introduce diagonal samples, i.e., let $\{\delta_{ii}\}_{i=1}^n$ be the i.i.d. copy of $\delta_{ij}$. Assume for the moment that $\tilde{\Omega}=\Omega\cup \{\delta_{ii}\}_{i=1}^n$, by adding and subtracting terms, we have
	\begin{align*}
		\setlength\belowdisplayskip{3pt}
		\setlength\abovedisplayskip{3pt}
		T_1& \leq \bigg\Vert(\frac{1}{p}\mathcal{P}_{\tilde{\Omega}}\mathbf{X}-\mathbf{X})\mathbf{1}\bigg\Vert_{\infty} + \bigg\Vert \mathrm{diag}(\frac{1}{p}\mathcal{P}_{\tilde{\Omega}}\mathbf{X}-\mathbf{X}) \bigg\Vert_{\infty}\\
		&\leq \bigg\Vert(\frac{1}{p}\mathcal{P}_{\tilde{\Omega}}\mathbf{X})\mathbf{1}\bigg\Vert_{\infty} + \frac{1}{p}\Vert\mathbf{X}\Vert_{\infty}\\
		& = \max_{m\in[n]}\bigg|\underbrace{\bigg\langle \frac{1}{p}\mathcal{P}_{\tilde{\Omega}}(\mathbf{11}^T),\mathbf{A}\bar{\mathbf{B}}^T \circ\mathbf{e}_m\mathbf{1}^T \bigg\rangle}_{T_{11}}\bigg| + \underbrace{\frac{1}{p}\Vert\mathbf{X}\Vert_{\infty}}_{T_{12}}.
	\end{align*}
	We first separate $\mathbf{A}$ in $T_{11}$, inspired by the following trick\cite[Thm. 4.1]{pmlr-v32-bhojanapalli14}
	\begin{align}
		\setlength\belowdisplayskip{3pt}
		\setlength\abovedisplayskip{3pt}
		T_{11} & = \bigg\langle \frac{1}{p}\mathcal{P}_{\tilde{\Omega}}(\mathbf{11}^T),\sum_{k=1}^r(\mathbf{A}_{\cdot,k}\bar{\mathbf{B}}_{\cdot,k}^T)\circ\mathbf{e}_m\mathbf{1}^T \bigg\rangle\nonumber\\
		& = \sum_{k=1}^r (\mathbf{A}_{\cdot,k}\circ \mathbf{e}_m)^T \frac{1}{p}\mathcal{P}_{\tilde{\Omega}}(\mathbf{11}^T) \bar{\mathbf{B}}_{\cdot,k}\nonumber\\
		& \overset{(i)}{\leq} \sqrt{\sum_{k=1}^r \mathbf{A}_{mk}^2}\sqrt{\sum_{k=1}^r \bigg|\mathbf{e}_m^T \frac{1}{p}\mathcal{P}_{\tilde{\Omega}}(\mathbf{11}^T)\bar{\mathbf{B}}_{\cdot,k}\bigg|^2}\nonumber\\
		\label{eq_T_11_seperateA}
		& \leq \Vert\mathbf{A}\Vert_{2,\infty} \underbrace{\sqrt{\sum_{k=1}^r \bigg|\mathbf{e}_m^T (\frac{1}{p}\mathcal{P}_{\tilde{\Omega}}(\mathbf{11}^T)-\mathbf{11}^T)\bar{\mathbf{B}}_{\cdot,k}\bigg|^2}}_{\bar{T}_{11}},
	\end{align}
	where $(i)$ from Cauchy-Schwarz. Set $\mathbf{G}=\frac{1}{p}\mathcal{P}_{\tilde{\Omega}}(\mathbf{11}^T)-\mathbf{11}^T$ for the moment, notice that
	\begin{align*}
		\setlength\belowdisplayskip{3pt}
		\setlength\abovedisplayskip{3pt}
		\bar{T}_{11} &= \sqrt{\mathbf{e}_m^T \mathbf{G}\sum_{k=1}^r(\bar{\mathbf{B}}_{\cdot,k}\bar{\mathbf{B}}_{\cdot,k}^T)\mathbf{G}^T\mathbf{e}_m}\\
		& = \sqrt{\mathbf{e}_m^T \mathbf{G}\bar{\mathbf{B}}\bar{\mathbf{B}}^T\mathbf{G}^T\mathbf{e}_m}=\Vert\mathbf{e}_m^T \mathbf{G}\mathbf{B}\boldsymbol{R}\Vert_2\\
		& = \bigg\Vert\sum_{j=1}^n (\frac{1}{p}\delta_{mj}-1)\mathbf{B}_{j,\cdot}^T\bigg\Vert_2=\bigg\Vert\sum_{j=1}^n \mathbf{S}_{mj}\bigg\Vert_2,
	\end{align*}
	holds for all $\boldsymbol{R}\in O(r)$. It is then straightforward to check 
	\begin{equation*}
		\setlength\belowdisplayskip{3pt}
		\setlength\abovedisplayskip{3pt}
		\mathbb{E} \mathbf{S}_{mj} = \mathbf{0},\,\Vert\mathbf{S}_{mj}\Vert_2\leq \frac{1}{p}\Vert\mathbf{B}\Vert_{2,\infty},
	\end{equation*}
	and 
	\begin{gather*}
		\setlength\belowdisplayskip{3pt}
		\setlength\abovedisplayskip{3pt}
		\bigg\Vert\sum_{j=1}^n \mathbb{E}\mathbf{S}_{mj}^T\mathbf{S}_{mj}\bigg\Vert \leq \frac{1}{p}\bigg\Vert\sum_{j=1}^n \mathbf{B}_{j,\cdot}\mathbf{B}_{j,\cdot}^T\bigg\Vert = \frac{1}{p}\Vert\mathbf{B}^T\mathbf{B}\Vert\leq\frac{1}{p}\Vert\mathbf{B}\Vert^2,\\
		\bigg|\sum_{j=1}^n \mathbb{E}\mathbf{S}_{mj}\mathbf{S}_{mj}^T\bigg| \leq \frac{1}{p} \bigg|\sum_{j=1}^n \mathbf{B}_{j,\cdot}^T\mathbf{B}_{j,\cdot}\bigg|\leq \frac{n}{p}\Vert\mathbf{B}\Vert_{2,\infty}^2.
	\end{gather*}
	Using elementary bound $\Vert\mathbf{B}\Vert^2\leq \Vert\mathbf{B}\Vert_F^2\leq n\Vert\mathbf{B}\Vert_{2,\infty}^2$ and Lemma \ref{eq_matrix_Bernstein}, we have
	\begin{equation}
		\setlength\belowdisplayskip{3pt}
		\setlength\abovedisplayskip{3pt}
		\label{eq_bar_T11_upper_final}
		\bar{T}_{11}\leq \sqrt{\frac{C\beta n\log n}{p}}\Vert\mathbf{B}\Vert_{2,\infty},
	\end{equation}
	holds with probability at least $1-n^{1-\beta}$ provided with $p\geq C\beta \log n/n$. By substituting \eqref{eq_bar_T11_upper_final} into \eqref{eq_T_11_seperateA}, it gives
	\begin{equation*}
		\setlength\belowdisplayskip{3pt}
		\setlength\abovedisplayskip{3pt}
		T_{11} \leq \sqrt{\frac{C\beta n\log n}{p}}\Vert\mathbf{A}\Vert_{2,\infty}\Vert\mathbf{B}\Vert_{2,\infty}.
	\end{equation*}
	Whenever $p\geq 1/n$, $T_{11}$ dominates $T_{12}$, after a union bound over $m\in [n]$, we conclude the whole proof.\qed
	\subsection{Proof of Lemma \ref{eq_NromalSpace_tight_bound}}
	\label{proof_of_lema_B5_NromalSpac}
	The original proof\cite[App. B-D]{LiSunSNLRCG2024ExArxiv} has some flaws, but we have amended it here. Notice that by \eqref{eq_Q_omega_s_stress_sample} and Cauchy-Schwarz, we have
	\begin{align}
		\setlength\belowdisplayskip{3pt}
		\setlength\abovedisplayskip{3pt}
		&\frac{1}{p}\Vert\mathcal{Q}_{\Omega}\Delta\Delta^T\Vert_F^2=\sum_{\boldsymbol{\alpha}\in\mathbb{I}}\frac{2}{p}\delta_{\boldsymbol{\alpha}}\langle\Delta\Delta^T,\boldsymbol{\omega}_{\boldsymbol{\alpha}}\rangle^2\nonumber\\
		&\leq\sum_{\boldsymbol{\alpha}\in\mathbb{I}}\frac{8}{p}\delta_{\boldsymbol{\alpha}}(\Vert\mathbf{e}_i^T\Delta\Vert_2^4+\Vert\mathbf{e}_j^T\Delta\Vert_2^4+2\Vert\mathbf{e}_i^T\Delta\Vert_2^2\Vert\mathbf{e}_j^T\Delta\Vert_2^2)\nonumber\\
		&=8\bigg\langle\sum_{\boldsymbol{\alpha}\in\mathbb{I}}\frac{1}{p}\delta_{\boldsymbol{\alpha}}\boldsymbol{\omega}_{\boldsymbol{\alpha}+},\mathbf{x}\mathbf{x}^T\bigg\rangle\nonumber\\
		&\leq 8\bigg\langle\sum_{\boldsymbol{\alpha}\in\mathbb{I}}(\frac{1}{p}\delta_{\boldsymbol{\alpha}}-1)\boldsymbol{\omega}_{\boldsymbol{\alpha}+},\mathbf{xx}^T \bigg\rangle+8\mathbf{x}^T(n\mathbf{I}+\mathbf{11}^T)\mathbf{x}\nonumber\\
		\label{eq_Fomega_DD_upper_1}
		&\leq 8\bigg\langle\sum_{\boldsymbol{\alpha}\in\mathbb{I}}\mathbf{B}_{\boldsymbol{\alpha}},\mathbf{xx}^T \bigg\rangle + 8n\Vert\mathbf{x}\Vert_2^2+ 8\Vert\mathbf{x}\Vert_1^2\nonumber\\
		&\leq (8\left\Vert\sum_{\boldsymbol{\alpha}\in\mathbb{I}}\mathbf{B}_{\boldsymbol{\alpha}}\right\Vert+ 8n)\Vert\mathbf{x}\Vert_2^2+ 8\Vert\mathbf{x}\Vert_1^2,
	\end{align} 
	where we set $\mathbf{x}=[\Vert\mathbf{e}_1^T\Delta\Vert_2^2,\dots,\Vert\mathbf{e}_n^T\Delta\Vert_2^2]^T$, $\boldsymbol{\omega}_{\boldsymbol{\alpha}+}=|\boldsymbol{\omega}_{\boldsymbol{\alpha}}|=(\mathbf{e}_i+\mathbf{e}_j)(\mathbf{e}_i+\mathbf{e}_j)^T$, and $\sum_{\boldsymbol{\alpha}\in\mathbb{I}}\boldsymbol{\omega}_{\boldsymbol{\alpha}+}=(n-2)\mathbf{I}+\mathbf{11}^T$. It remains to control the spectral norm of $\sum_{\boldsymbol{\alpha}\in\mathbb{I}}(\frac{1}{p}\delta_{\boldsymbol{\alpha}}-1)\boldsymbol{\omega}_{\boldsymbol{\alpha}+}:=\sum_{\boldsymbol{\alpha}\in\mathbb{I}}\mathbf{B}_{\boldsymbol{\alpha}}$. It is straightforward to show
	\begin{equation*}
		\setlength\belowdisplayskip{3pt}
		\setlength\abovedisplayskip{3pt}
		\mathbb{E}\mathbf{B}_{\boldsymbol{\alpha}}=\mathbf{0},\,\Vert\mathbf{B}_{\boldsymbol{\alpha}}\Vert\leq \frac{1}{p}\Vert\boldsymbol{\omega}_{\boldsymbol{\alpha}+}\Vert\leq\frac{2}{p},
	\end{equation*}
	and
	\begin{align*}
		\setlength\belowdisplayskip{3pt}
		\setlength\abovedisplayskip{3pt}
		\bigg\Vert\mathbb{E}\sum_{\boldsymbol{\alpha}\in\mathbb{I}}\mathbf{B}_{\boldsymbol{\alpha}}^2\bigg\Vert&\leq\frac{1}{p}\bigg\Vert\sum_{\boldsymbol{\alpha}\in\mathbb{I}}\boldsymbol{\omega}_{\boldsymbol{\alpha}+}^2\bigg\Vert\leq\frac{4}{p}\bigg\Vert\sum_{\boldsymbol{\alpha}\in\mathbb{I}}\boldsymbol{\omega}_{\boldsymbol{\alpha}+}\bigg\Vert\\
		&\leq \frac{4}{p}\Vert n\mathbf{I}+\mathbf{11}^T\Vert\leq\frac{8n}{p},
	\end{align*}
	since $\boldsymbol{\omega}_{\boldsymbol{\alpha}+}^2\preccurlyeq4\boldsymbol{\omega}_{\boldsymbol{\alpha}+},\,\forall\,\boldsymbol{\alpha}\in\mathbb{I}$. From Lemma \ref{eq_matrix_Bernstein}, we have
	\begin{equation}
		\setlength\belowdisplayskip{3pt}
		\setlength\abovedisplayskip{3pt}
		\label{eq_Balpha_distoration_bound}
		\bigg\Vert\sum_{\boldsymbol{\alpha}\in\mathbb{I}}(\frac{1}{p}\delta_{\boldsymbol{\alpha}}-1)\boldsymbol{\omega}_{\boldsymbol{\alpha}+}\bigg\Vert\leq\sqrt{\frac{c\beta n\log n}{p}},
	\end{equation}
	holds with probability at least $1-n^{1-\beta}$ given $p\geq C\beta\log n/n$. By substituting \eqref{eq_Balpha_distoration_bound} into \eqref{eq_Fomega_DD_upper_1}, it gives
		\begin{align*}
			\setlength\belowdisplayskip{3pt}
			\setlength\abovedisplayskip{3pt}
			\frac{1}{p}&\Vert\mathcal{Q}_{\Omega}\Delta\Delta^T\Vert_F^2\leq (\sqrt{\frac{c\beta n\log n}{p}}+8n)\Vert\mathbf{x}\Vert_2^2+8\Vert\mathbf{x}\Vert_1^2\\
			&\leq \bigg[(\sqrt{\frac{c\beta n\log n}{p}}+8n)\Vert\Delta\Vert_{2,\infty}^2 + 8 \Vert\Delta\Vert_F^2\bigg]\Vert\Delta\Vert_F^2,
		\end{align*}
	holds for any $\Delta$, thus concluding the proof. \qed
}
\bibliographystyle{IEEEtran}
\bibliography{IEEEabrv,reference}

\end{document}